%% file: jellyfish.tex
\documentclass[nohyperref]{article}

\usepackage{microtype}
\usepackage{graphicx}
\usepackage{subfigure}
\usepackage{booktabs} % for professional tables

\usepackage{breakurl}

\usepackage{hyperref}

\usepackage[accepted]{icml2023}

\usepackage{amsmath}
\usepackage{amssymb}
\usepackage{mathtools}
\usepackage{amsthm}

\usepackage[capitalize,noabbrev]{cleveref}

\usepackage{breakurl}

\DeclareMathOperator{\End}{End}

\DeclareMathOperator{\Hom}{Hom}

\DeclareMathOperator{\Bell}{B}

\DeclareMathOperator{\sgn}{sgn}
\DeclareMathOperator{\Stab}{Stab}

\usepackage{mathdots}
\usepackage{bm}
\usepackage{nicematrix}
\usepackage{xcolor}

\usepackage{tabularray}
\usepackage{comment}

\usepackage[title]{appendix}
\usepackage{appendix}
\usepackage{xurl}

\usepackage{ytableau}
\usepackage{tikzit}
\input{mystyle.tikzstyles}
\usepackage{graphicx}

\theoremstyle{plain}
\newtheorem{theorem}{Theorem}[section]
\newtheorem{proposition}[theorem]{Proposition}
\newtheorem{lemma}[theorem]{Lemma}

\theoremstyle{definition}
\newtheorem{definition}[theorem]{Definition}

\theoremstyle{remark}
\newtheorem{remark}[theorem]{Remark}

\usepackage[textsize=tiny]{todonotes}

\icmltitlerunning{
	How Jellyfish Characterise Alternating Group Equivariant Neural Networks
}

\begin{document}

\twocolumn[
\icmltitle{
	How Jellyfish Characterise Alternating Group Equivariant Neural Networks
}




\begin{icmlauthorlist}
	\icmlauthor{Edward Pearce--Crump}{imperial}
\end{icmlauthorlist}

\icmlaffiliation{imperial}{Department of Computing, Imperial College London, United Kingdom}

\icmlcorrespondingauthor{Edward Pearce--Crump}{ep1011@ic.ac.uk}

\icmlkeywords{Machine Learning, ICML}

\vskip 0.3in
]



\printAffiliationsAndNotice{}  % leave blank if no need to mention equal contribution

\begin{abstract}
	We provide a full characterisation of all of the possible alternating group ($A_n$) equivariant neural networks 
	whose layers are some tensor power of $\mathbb{R}^{n}$.
	In particular, 
	we find a basis of matrices for the learnable, linear, $A_n$--equivariant layer functions between such tensor power spaces
	in the standard basis of $\mathbb{R}^{n}$.
	%The neural networks that we characterise are simple to implement since our method circumvents the typical requirement when building group equivariant neural networks of having to decompose the representations into irreducibles.
	We also describe how our approach generalises to the construction of neural networks that are equivariant to local symmetries.
\end{abstract}

\section{Introduction}

There has been a growing research effort in deep learning to develop
%Research in deep learning has focused on developing 
neural network architectures that can be used to learn efficiently from data that possesses an underlying symmetry. 
These architectures guarantee that the functions that are learned are subject to a geometric property,
known as \textit{equivariance},
that is connected with the symmetry group.
Group equivariant neural networks are important due to the additional advantages that they offer over traditional multilayer perceptron models. 
For example, they commonly display high levels of parameter sharing within each layer, resulting in significantly fewer parameters overall. 
This often leads to models that show improved prediction performance on new data.

\begin{comment}
Creating neural network architectures that can be used to learn efficiently from data that has an underlying symmetry 
%baked into it 
has become an active area of research in deep learning. 
All of these architectures 
%that have been found 
%are all guided by 
guarantee 
that the functions that can be learned have
a fundamental geometric property in relation to the symmetry group that is involved: namely, \textit{equivariance}.
The importance of group equivariant neural networks is in part owing to the additional benefits that they offer over their standard multilayer perceptron counterparts.
%Group equivariant neural networks are important since they
%have been shown to offer additional benefits over their standard multilayer perceptron counterparts: 
Typically, they exhibit high levels of parameter sharing within each layer, resulting in far fewer parameters overall. 
This often leads to models that have better predictive performance on unseen data.
\end{comment}

The symmetry group that has received the most attention, 
in terms of it
%relation to it 
being explicitly incorporated into 
%machine learning 
neural network architectures,
is the group of all permutations on some fixed number of objects, called the symmetric group.
%Much of the focus has been on understanding 
%It is desirable to build
%Knowing how to build 
Creating neural networks that are equivariant to permutations is highly desirable as many data structures, such as sets and graphs, exhibit natural permutation symmetry.
%We want to build
%neural networks that are equivariant to permutations 
%%is important 
%since many data structures, such as sets and graphs, naturally exhibit permutation symmetry.
It is easily understood that the labelling of the elements of a set or the vertices of a graph is arbitrary; 
hence, it is crucial to ensure that the functions that are learned from such data 
%with permutation symmetry 
do not depend on how the data is labelled.
%the labelling of the data.

%hence it is important to 
%%be able to build neural networks that 
%guarantee that the functions that can be learned from such data do not depend on its labelling. 

In this paper, 
we look instead at how to construct neural networks that are equivariant to the alternating group.
The alternating group is an index two subgroup of the symmetric group consisting solely
%of all 
of all 
of the 
even permutations.
%in the symmetric group.
Alternating group symmetry has proven to be particularly useful when learning from spherical image data that has been discretely represented on an icosahedron
\cite{zhang};
in constructing convolutional neural networks on an icosahedron
\cite{cohen19d};
and in estimating polynomials that are invariant to the action of the alternating group
\cite{kicki}.

%In this paper, 
%In particular,
Specifically,
we give a full characterisation of all of the possible alternating group equivariant neural networks whose layers are some tensor power of $\mathbb{R}^{n}$ by finding a basis of matrices for the learnable, linear, alternating group equivariant layer functions between such tensor power spaces in the standard basis of $\mathbb{R}^{n}$.

%This makes the 
%%alternating group equivariant 
%neural networks in question simple to implement.

Our approach is similar the one presented in the papers written by
Pearce--Crump~\yrcite{pearcecrump, pearcecrumpB}.
They used different sets of set partition diagrams to characterise
all of the learnable, linear, group equivariant layer functions 
between tensor power spaces in the standard basis of $\mathbb{R}^{n}$
for the following groups:
the symmetric group $S_n$; the orthogonal group $O(n)$; the symplectic group $Sp(n)$; and the special orthogonal group $SO(n)$.
We will show that in the case of the alternating group $A_n$,
%learnable, linear, alternating group equivariant layer functions 
%between tensor power spaces in the standard basis of $\mathbb{R}^{n}$
the layer functions can also be characterised by
%based on 
certain
%specific 
sets of set partition diagrams.

To do this, we use a concept that was first introduced by Comes~\yrcite{comes}, namely, so-called \textit{jellyfish}.
In their paper, they largely determined the theory of alternating group equivariance; however, they relied heavily on the language of category theory in their exposition.
We simplify their approach, and provide proofs that are more accessible 
to 
%those in
%e general reader in
%to 
the machine learning community.
%than those that are given in their paper.

The main contributions of this paper, which appear in Section \ref{altgroupsection} onwards, are as follows:
\begin{enumerate}
	\item We are the first to show how the combinatorics underlying set partition diagrams, together with some jellyfish, serves as the theoretical foundation for constructing neural networks that are equivariant to the alternating group when the layers are some tensor power of $\mathbb{R}^{n}$.
	\item % Specifically
		In particular, we find a basis for the learnable, linear, $A_n$--equivariant layer functions between such tensor power spaces in the standard basis of $\mathbb{R}^{n}$.
	\item We extend our approach to show how to construct neural networks that are equivariant to local symmetries.
\end{enumerate}

\begin{comment}
Order of work:
1. recall the $S_n$ orbits and how they link to partition diagrams with up to $n$ blocks.
Recall the flattening procedure too ...
2. describe how to get the $A_n$ orbits + dimension formulae.
3. describe how to adapt the appropriate diagrams to "identify" the $A_n$ orbits for $n-1$ and $n$ block partition diagrams.
(suggest that this gives you a Schur--Weyl duality between jellyfish partition algebra and ??)
4. Add in proofs for these things.
5. Add in for features, biases (use my nice short explanation from IJCAI submission for Brauer paper)
6. Generalise to local symmetries (again, use my nice short explanation for Brauer paper)
7. Short section doing a summary of partitions, brauer, and jellyfish in terms of category theory (and string diagrams) and how you get group equivariant neural networks this way ... suggesting a new line of approach (adding appropriate references)!
8. Examples?
\end{comment}

\section{Preliminaries}

We choose our field of scalars to be $\mathbb{R}$ throughout. 
Tensor products are also taken over $\mathbb{R}$, unless otherwise stated.
Also, we let $[n]$ represent the set $\{1, \dots, n\}$. 

Recall that a representation of a group $G$ is a choice of vector space $V$ over $\mathbb{R}$ and a group homomorphism
\begin{equation} \label{grouprephom}
	\rho : G \rightarrow GL(V)	
\end{equation}
We choose to focus on finite-dimensional vector spaces $V$ 
that are some tensor power of $\mathbb{R}^{n}$
in this paper.

We often abuse our terminology by calling $V$ a representation of $G$, even though the representation is technically the homomorphism $\rho$.
When the homomorphism $\rho$ needs to be emphasised alongside its vector space $V$, we will use the notation $(V, \rho)$.

\section{$(\mathbb{R}^{n})^{\otimes k}$, a representation of both $S_n$ and $A_n$}

Recall that $S_n$ is the group of all permutations on $[n]$, and that $A_n$ is the subgroup of $S_n$ consisting of all permutations on $[n]$ whose image under the function
\begin{equation}
	\sgn : S_n \rightarrow \{\pm 1\}
\end{equation}
is $+1$, where $\sgn$ is defined, for all $\sigma \in S_n$, as
\begin{equation}
	\sgn(\sigma) 
	\coloneqq
	\begin{cases}
		+1 & \text{if $\sigma$ is an even permutation in $S_n$} \\
		-1 & \text{if $\sigma$ is an odd permutation in $S_n$} \\
	\end{cases}
\end{equation}
We have that $\mathbb{R}^{n}$ is a representation of $S_n$ via its (left) action on the basis $\{e_a \mid a \in [n]\}$ which is extended linearly,
where, specifically, the action is given by
\begin{equation}
	\sigma \cdot e_a = e_{\sigma(a)} \text{ for all } \sigma \in S_n \text{ and } a \in [n]
\end{equation}
Restricting this action to $A_n$ and extending linearly shows that $\mathbb{R}^{n}$ is also a representation of $A_n$.

We also have that any $k$-tensor power of $\mathbb{R}^{n}$, $(\mathbb{R}^{n})^{\otimes k}$, for any $k \in \mathbb{Z}_{\geq 0}$, is a representation of $S_n$,
since the elements 
\begin{equation} \label{tensorelementfirst}
	e_I \coloneqq e_{i_1} \otimes e_{i_2} \otimes \dots \otimes e_{i_k} 
\end{equation}
for all $I \coloneqq (i_1, i_2, \dots, i_k) \in [n]^k$ form a basis of 
$(\mathbb{R}^{n})^{\otimes k}$,
and the action of $S_n$ that maps a basis element of 
$(\mathbb{R}^{n})^{\otimes k}$
of the form (\ref{tensorelementfirst}) to
\begin{equation}
	e_{\sigma(I)} \coloneqq e_{\sigma(i_1)} \otimes e_{\sigma(i_2)} \otimes \dots \otimes e_{\sigma(i_k)} 
\end{equation}
can be extended linearly. 
Again, restricting this action to $A_n$ and extending linearly shows that 
$(\mathbb{R}^{n})^{\otimes k}$
is also a representation of $A_n$.

We denote the representation of $S_n$ by $\rho_{k}$. 
We will use the same notation for the restriction of this representation to $A_n$, with the context making clear that it is the restriction of the $S_n$ representation.

For more on the representation theory of the symmetric and alternating groups, see \cite{sagan} and \cite{tolli}.

\section{Group Equivariant Neural Networks} \label{Groupequivnnssection}

Group equivariant neural networks are constructed by alternately composing linear and non-linear $G$-equivariant maps between representations of a group $G$. 
%For more details, see 
The following is based on the material presented in \cite{lim}.
%See \cite{lim} for more details.

We first define \textit{$G$-equivariance}:
%below:
\begin{definition} \label{Gequivariance}
	Suppose that $(V, \rho_{V})$ and $(W, \rho_{W})$
	are two representations of a group $G$.

	A map $\phi : V \rightarrow W$ is said to be $G$-equivariant if,
	for all $g \in G$ and $v \in V$,
\begin{equation} \label{Gequivmapdefn}
	\phi(\rho_{V}(g)[v]) = \rho_{W}(g)[\phi(v)]
	%\text{ for all } g \in G \text{ and } v \in V.
\end{equation}
	The set of all \textit{linear} $G$-equivariant maps between $V$ and $W$ is denoted by $\Hom_{G}(V,W)$. 
	When $V = W$, we write this set as $\End_{G}(V)$.
	It can be shown that $\Hom_{G}(V,W)$ is a vector space over $\mathbb{R}$, and that $\End_{G}(V)$ is an algebra over $\mathbb{R}$.
	See \cite{segal} for more details. 
\end{definition}

A special case of $G$-equivariance is \textit{$G$-invariance}:
\begin{definition}
	The map $\phi$ given in Definition \ref{Gequivariance} is said to be $G$-invariant if $\rho_{W}$ is defined to be the $1$-dimensional trivial representation of $G$.
	As a result, $W = \mathbb{R}$.
	%\textbf{CHECK THIS!}
\end{definition}

%\begin{remark}
	%Since the $d$--dimensional trivial representation of $O(n)$ can be decomposed into a direct sum of $d$ irreducible representations of $O(n)$, each of 
	%which is the irreducible $1$--dimensional trivial representation of $O(n)$, we choose in the following to focus on $O(n)$--invariant functions that map to the $1$--dimensional trivial representation $\mathbb{R}$. \textbf{CHECK THIS!}
%\end{remark}

We can now define the type of neural network that is the focus of this paper:
\begin{definition} \label{Gneuralnetwork}
	An $L$-layer $G$-equivariant neural network $f_{\mathit{NN}}$ is a composition of \textit{layer functions}
	\begin{equation}
		f_{\mathit{NN}} \coloneqq f_L \circ \ldots \circ f_{l} \circ \ldots \circ f_1
  	\end{equation}
	such that the $l^{\text{th}}$ layer function is a map of representations of $G$
	\begin{equation}
		f_l: (V_{l-1}, \rho_{l-1}) \rightarrow (V_l, \rho_l)
	\end{equation}
	that is itself a composition
  	\begin{equation} \label{Glayerfidefn}
	  	f_l \coloneqq \sigma_l \circ \phi_l
  	\end{equation}
	of a learnable, linear, $G$-equivariant function
	\begin{equation} \label{Glayerlinear}
		\phi_l : (V_{l-1}, \rho_{l-1}) \rightarrow (V_l, \rho_l)
	\end{equation}
	together with 
	a fixed, non-linear activation function 
	\begin{equation}	
		\sigma_l: (V_l, \rho_l) \rightarrow (V_l, \rho_l)
	\end{equation}
	such that
	\begin{enumerate}
		\item $\sigma_l$ is a $G$-equivariant map, as in (\ref{Gequivmapdefn}), and
		\item $\sigma_l$ acts pointwise (after a basis has been chosen for each copy of $V_l$ in $\sigma_l$.)
	\end{enumerate}
\end{definition}

We focus on the learnable, linear, $G$-equivariant functions in this paper, since the non-linear functions are fixed. 
Note that, after picking a basis for each layer space in the set $\{V_l\}$, 
the number of parameters 
that appear
in a matrix representation of a learnable, linear, $G$-equivariant function of the form (\ref{Glayerlinear}), that is, in a weight matrix for the $l^{\text{th}}$ layer function, 
is equal to the number of matrices that appear in a basis for
$\Hom_{G}(V_{l-1},V_l)$.
Furthermore, 
given a basis of
$\Hom_{G}(V_{l-1},V_l)$,
the weight matrix itself is a weighted linear combination of these basis matrices, where each coefficient in the linear combination is a parameter to be learned.

\begin{remark}
	The entire neural network $f_{\mathit{NN}}$ is itself a $G$-equivariant function because 
	it can be shown that
	the composition of any number of $G$-equivariant functions is itself $G$-equivariant.
\end{remark}

\begin{remark}
	One way of making 
	a neural network of the form given in Definition \ref{Gneuralnetwork}
	$G$-invariant 
	is by choosing the representation in the final layer
	to be the $1$-dimensional trivial representation of $G$. 
\end{remark}

\section{Symmetric Group}

In this section, we recall the technique 
of using set partitions 
to find a basis of 
%for finding a basis of 
	$\Hom_{S_n}((\mathbb{R}^{n})^{\otimes k}, (\mathbb{R}^{n})^{\otimes l})$
in the standard basis of $\mathbb{R}^{n}$
since it will feature heavily in what follows for the alternating group.
For more details, see~\cite{maron2018}, \cite{ravanbakhsh}, and
\cite{pearcecrump}.

\subsection{Set Partitions} \label{setpartitions}

%We introduce set partitions since, as we will show, certain set partitions of $[l+k]$ correspond bijectively to the basis elements of 
%	$\Hom_{S_n}((\mathbb{R}^{n})^{\otimes k}, (\mathbb{R}^{n})^{\otimes l})$.

For $l, k \in \mathbb{Z}_{\geq 0}$, consider the set $[l+k]$ having $l+k$ elements.
We can create a set partition of $[l+k]$ by partitioning it into a number of subsets.
We call the subsets of a set partition \textit{blocks}.
Define $\Pi_{l+k}$ to be the set of all set partitions of $[l+k]$. 
Let $\Pi_{l+k, n}$ be the subset of $\Pi_{l+k}$ consisting of all set partitions of $[l+k]$ having at most $n$ blocks.

As the number of set partitions in $\Pi_{l+k}$ having exactly $t$ blocks is the Stirling number
$
\begin{Bsmallmatrix}
l+k\\
t 
\end{Bsmallmatrix}
$
of the second kind, we see that the number of elements in $\Pi_{l+k}$ is equal to $\Bell(l+k)$, the $(l+k)^{\text{th}}$ Bell number, and that the number of elements in $\Pi_{l+k,n}$ is therefore equal to
\begin{equation}
	\sum_{t=1}^{n} 
		\begin{Bmatrix}
		l+k\\
		t 
		\end{Bmatrix}
	\coloneqq \Bell(l+k,n)
\end{equation}
the $n$-restricted $(l+k)^{\text{th}}$ Bell number.

For each set partition $\pi$ in $\Pi_{l+k}$, we can define a diagram $d_{\pi}$ 
that has two rows of vertices and edges between vertices such that there are
\begin{enumerate}
	\item $l$ vertices on the top row, labelled by $1, \dots, l$
	\item $k$ vertices on the bottom row, labelled by $l+1, \dots, l+k$, and
	\item 
		the edges between the vertices are such that the connected components of $d_\pi$ correspond to the blocks of $\pi$.
		%the edges between the vertices correspond to the connected components of $\pi$.
		%the set partition that indexes the diagram.
\end{enumerate}
%As a result
In particular, $d_{\pi}$ represents the equivalence class of all diagrams with connected components equal to the blocks of $\pi$. 

\begin{figure}[t]
	\begin{center}
	\scalebox{0.5}{\input{partklelement1.tikz}}
		\caption{A diagram $d_{\pi}$ corresponding to the set partition $\pi$ given in (\ref{examplesetpart}) with $l=3$ and $k=5$.}
	\label{partklelement1}
	\end{center}
\end{figure}

For example, if $l = 3$ and $k = 5$, a diagram 
corresponding to the set partition 
\begin{equation} \label{examplesetpart}
	\pi \coloneqq \{1, 5 \mid 2, 4 \mid 3, 6, 8 \mid 7 \}
\end{equation}
in $\Pi_{8}$ consisting of $4$ blocks is given in Figure \ref{partklelement1}.

We can define another diagram related to each set partition 
diagram $d_\pi$:
%$\pi$ in $\Pi_{l+k}$: 
it is the flattened version of $d_\pi$, where we pull the top row of $l$ vertices down and to the left of the bottom row of $k$ vertices, maintaining the order of the labels.

For example, the flattened diagram 
corresponding to the set partition diagram shown in Figure \ref{partklelement1} is
%(\ref{examplesetpart}) is
given in Figure \ref{partklelement2}.

\begin{figure}[t]
	\begin{center}
	\scalebox{0.5}{\input{partklelement2.tikz}}
		\caption{The flattened diagram corresponding to 
		the set partition diagram shown in Figure \ref{partklelement1}.}
		%the set partition $\pi$ given in (\ref{examplesetpart}).}
	\label{partklelement2}
	\end{center}
\end{figure}

\subsection{
A Basis of 
	$\Hom_{S_n}((\mathbb{R}^{n})^{\otimes k}, (\mathbb{R}^{n})^{\otimes l})$
}

%\textbf{TO DO:} Recall symmetric group result here, starting with the $S_n$ orbits on $[n]^{l+k}$, and that they are in bijective correspondence with the set partitions of $[l+k]$ having at most $n$ blocks.

%The key idea in what follows is that
%the basis elements of 
	%$\Hom_{S_n}((\mathbb{R}^{n})^{\otimes k}, (\mathbb{R}^{n})^{\otimes l})$
%are in bijective correspondence with the orbits coming from the action of $S_n$ on $[n]^{l+k}$, and that each such orbit corresponds bijectively to a set partition of $\Pi_{l+k,n}$.

%To show this, 
We begin by noting that
	$\Hom((\mathbb{R}^{n})^{\otimes k}, (\mathbb{R}^{n})^{\otimes l})$
has a standard basis of matrix units
\begin{equation}
	\{E_{I,J}\}_{I \in [n]^l, J \in [n]^k}
\end{equation}
where $E_{I,J}$ has a $1$ in the $(I,J)$ position and is $0$ elsewhere.

Hence, 
expressing $f \in
	\Hom((\mathbb{R}^{n})^{\otimes k}, (\mathbb{R}^{n})^{\otimes l})$
in this basis as
%the basis of matrix units as
\begin{equation}
	f = \sum_{I \in [n]^l}\sum_{J \in [n]^k} f_{I,J}E_{I,J}
\end{equation}
it can be shown that
	$f \in \Hom_{S_n}((\mathbb{R}^{n})^{\otimes k}, (\mathbb{R}^{n})^{\otimes l})$ if and only if 
\begin{equation} \label{centeq1}
	f_{\sigma(I),\sigma(J)} = f_{I,J} 
\end{equation}
for all $\sigma \in S_n$ and $I \in [n]^{l}, J \in [n]^{k}$.

Concatenating each $I \in [n]^l, J \in [n]^{k}$ into a single element $(I, J) \in [n]^{l+k}$, 
we see from (\ref{centeq1}) that the basis elements of 
$\Hom_{S_n}((\mathbb{R}^{n})^{\otimes k}, (\mathbb{R}^{n})^{\otimes l})$
are in bijective correspondence with the orbits coming from the action of $S_n$ on $[n]^{l+k}$, where $\sigma \in S_n$ acts on the pair $(I,J)$ by 
\begin{equation} \label{sigmapairaction}
	\sigma(I,J) \coloneqq (\sigma(I),\sigma(J))
\end{equation}
%Moreover, since 
As $S_n$ acts on $[n]$ transitively, it acts on $[n]^{l+k}$ transitively, meaning that the orbits under this action actually partition the set $[n]^{l+k}$ into equivalence classes.

Furthermore, we can define a bijection between the orbits 
coming from the action of $S_n$ on $[n]^{l+k}$
and the set partitions $\pi$ in $\Pi_{l+k}$ having at most $n$ blocks.
Indeed, if $(I,J)$ is a class representative of an orbit, then,
%Now, to show that each orbit corresponds bijectively to a set partition $\pi$ in $\Pi_{l+k}$ having at most $n$ blocks, we define a bijection
%on a class representative $(I,J)$ of an orbit as follows:
replacing momentarily the elements of $J$ by $i_{l+m} \coloneqq j_m$ for all $m \in [k]$,
so that
\begin{align} \label{pairedblocks}
	(I,J) 
	& = (i_1, i_2, \dots, i_l, j_1, j_2, \dots, j_k) \nonumber \\
	& = (i_1, i_2, \dots, i_l, i_{l+1}, i_{l+2}, \dots i_{l+k})
\end{align}
we define the bijection by
\begin{equation} \label{blockbij}
	i_x = i_y \iff x, y \text{ are in the same block of } \pi	
\end{equation}
for all $x$, $y \in [l+k]$.

The bijection (\ref{blockbij}) is independent of the choice of class representative since
\begin{equation}
	i_x = i_y \iff \sigma(i_x) = \sigma(i_y) \text{ for all } \sigma \in S_n
\end{equation}
Notice that the LHS of (\ref{blockbij}) is checking for an equality on the elements of $[n]$, whereas the RHS is separating the elements of $[l+k]$ into blocks; hence $\pi$ must have at most $n$ blocks.

As a result, we have shown the following.

\begin{theorem} \label{Sncorresp}
The basis elements of 
$\Hom_{S_n}((\mathbb{R}^{n})^{\otimes k}, (\mathbb{R}^{n})^{\otimes l})$,
in the standard basis of $\mathbb{R}^{n}$,
correspond bijectively with all set partitions $\pi$ in $\Pi_{l+k}$ having at most $n$ blocks, 
	which correspond bijectively with the orbits coming from the action of $S_n$ on $[n]^{l+k}$.
\end{theorem}

We can obtain the basis elements themselves, as follows:

Taking a set partition $\pi$ in $\Pi_{l+k,n}$,
and denoting the number of blocks in $\pi$ by $t$, we obtain a labelling of the blocks by letting $B_1$ be the block that contains the number $1 \in [l+k]$, and iteratively letting $B_j$, for $1 < j \leq t$, be the block that contains the smallest number in $[l+k]$ that is not in $B_1 \cup B_2 \cup \dots \cup B_{j-1}$.

Using this labelling of the blocks, we can create an element of $[n]^{l+k}$ by letting the $i^{\text{th}}$ position, $i \in [l+k]$, be the label of the block containing the number $i$.
This is called the \textit{block labelling} of $\pi$.
Denote it by the pair $(I_{\pi},J_{\pi})$, where clearly $I_{\pi} \in [n]^l$ and $J_{\pi} \in [n]^{k}$.

In doing so, we have created a representative of the orbit for the $S_n$ action on $[n]^{l+k}$ corresponding to the set partition $\pi \in \Pi_{l+k,n}$
under the bijection given in (\ref{blockbij}). 
Denote this orbit by $O_{S_n}((I_{\pi},J_{\pi}))$.

%By obtaining all pairs $(I,J)$ appearing in this orbit, 
We form a basis element of 
$\Hom_{S_n}((\mathbb{R}^{n})^{\otimes k}, (\mathbb{R}^{n})^{\otimes l})$,
denoted by $X_{\pi}$, 
by adding together all matrix units 
%indexed by each pair 
whose indexing pair $(I,J)$ appears
in the orbit $O_{S_n}((I_{\pi},J_{\pi}))$; that is,
\begin{equation} \label{equivbasiselement}
	%\sum_{(P,Q) \in O((I,L)_{\pi})} E_{P,Q}
	X_\pi \coloneqq \sum_{(I,J) \in O_{S_n}((I_{\pi},J_{\pi}))} E_{I,J}
\end{equation}
We see that $X_\pi$ is a basis element of 
$\Hom_{S_n}((\mathbb{R}^{n})^{\otimes k}, (\mathbb{R}^{n})^{\otimes l})$
by (\ref{centeq1}).

Doing this for each set partition $\pi$ in $\Pi_{l+k,n}$, or, equivalently, for all of the orbits coming from the action of $S_n$ on $[n]^{l+k}$, gives:
\begin{theorem}
	%[\cite{pearcecrump}]
	\label{equivbasislplusk}
	For $l, k \in \mathbb{Z}_{\geq 0}, n \in \mathbb{Z}_{\geq 1}$,
	we have that
	\begin{equation}
		\{X_\pi \mid \pi \in \Pi_{l+k,n}\}
	\end{equation}
	is a basis of 
	$\Hom_{S_n}((\mathbb{R}^{n})^{\otimes k}, (\mathbb{R}^{n})^{\otimes l})$,
	and so 
	\begin{align}
		\dim 
		\Hom_{S_n}((\mathbb{R}^{n})^{\otimes k}, (\mathbb{R}^{n})^{\otimes l})
		& = 
		%\sum_{t = 1}^{n} 
			%\begin{Bmatrix}
				%l+k\\
				%t 
			%\end{Bmatrix}
		%\nonumber \\
		%& =
		\Bell(l+k, n)
	\end{align}
	%where $\Bell(l+k,n)$ is the $n$-restricted $(l+k)^{\text{th}}$ Bell number.
\end{theorem}

\section{Alternating Group} \label{altgroupsection}

In exactly the same way as for the symmetric group, we can show that the basis elements of 
	$\Hom_{A_n}((\mathbb{R}^{n})^{\otimes k}, (\mathbb{R}^{n})^{\otimes l})$
	are in bijective correspondence with the orbits coming from the action of $A_n$ on $[n]^{l+k}$.

However, the major difference 
between the symmetric group and the alternating group 
is that the $A_n$ orbits on $[n]^{l+k}$, and consequently the basis elements of 
	$\Hom_{A_n}((\mathbb{R}^{n})^{\otimes k}, (\mathbb{R}^{n})^{\otimes l})$,
are not necessarily in bijective correspondence with the set partitions of $[l+k]$ having at most $n$ blocks.
This is because some of the $S_n$ orbits on $[n]^{l+k}$ become the disjoint union of more than one $A_n$ orbit on $[n]^{l+k}$.
We say that an $S_n$ orbit on $[n]^{l+k}$ \textit{splits} if it is the disjoint union of more than one $A_n$ orbit on $[n]^{l+k}$. 
Our first task is to identify which $S_n$ orbits on $[n]^{l+k}$ split and which do not.

\begin{theorem} [$S_n$ orbits split] \label{SnOrbitSplit}
	%Let $k, l \in \mathbb{Z}_{\geq 0}$ and $n \in \mathbb{Z}_{\geq 1}$.

	Let $\pi$ be a set partition in $\Pi_{l+k,n}$ having some $t$ blocks, and consider its corresponding $S_n$ orbit on $[n]^{l+k}$,
	$O_{S_n}((I_\pi, J_\pi))$.

	If $t \leq n-2$, then $O_{S_n}((I_\pi, J_\pi))$ does not split.

	If $t = n-1$ or $n$, then $O_{S_n}((I_\pi, J_\pi))$ splits into a disjoint union of exactly two $A_n$ orbits.
\end{theorem}

\begin{proof} 
	Consider the case where $\pi$ has $t = n$ blocks. 
	Choose some arbitrary element $(I_\alpha, J_\alpha) \in O_{S_n}((I_\pi, J_\pi))$, and 
	%We look at 
	consider
	$\Stab_{A_n}((I_\alpha, J_\alpha))$,
	the stabilizer of $(I_\alpha, J_\alpha)$ under the action of $A_n$.
	Since all of the elements in $[n]$ appear in the tuple $(I_\alpha, J_\alpha)$, 
	the only element of $A_n$ that fixes the entries of $(I_\alpha, J_\alpha)$
	is the identity element; that is, $\Stab_{A_n}((I_\alpha, J_\alpha)) = \{e_{A_n}\}$.
	%Since the only element in $A_n$ that fixes all of the $n$ different entries in $(I_\alpha, J_\alpha)$ is the identity element, we have, 
	Hence, by the Orbit--Stabilizer Theorem,
	%(\ref{orbstabAn}), 
	we have that
	$\left|O_{A_n}((I_\alpha, J_\alpha))\right| = \frac{n!}{2}$.
	Since 
	$\left|O_{S_n}((I_\pi, J_\pi))\right| = n!$ in this case,
	%by (\ref{orbstabSn}), 
	and because
	$(I_\alpha, J_\alpha)$ was an arbitary element of $O_{S_n}((I_\pi, J_\pi))$, 
	we see that,
	by another application of the Orbit--Stabilizer Theorem,
	$O_{S_n}((I_\pi, J_\pi))$ splits into a disjoint union of two $A_n$ orbits.

	Similarly, for the case where $\pi$ has $t = n-1$ blocks, we see that 
	$\Stab_{A_n}((I_\alpha, J_\alpha)) = \{e_{A_n}\}$,
	and so, by the same argument,  
	%(\ref{orbstabSn}) and (\ref{orbstabAn}), 
	%we have that
	$O_{S_n}((I_\pi, J_\pi))$ splits into a disjoint union of two $A_n$ orbits.

	Finally, for $t \leq n-2$, we have that $\Stab_{A_n}((I_\alpha, J_\alpha)) \cong A_{n-t}$, and so, by the Orbit--Stabilizer Theorem,
	%(\ref{orbstabSn}) and (\ref{orbstabAn}), 
	we see that
	$\left|O_{A_n}((I_\alpha, J_\alpha))\right| = 
	\left|O_{S_n}((I_\pi, J_\pi))\right|$.
	Consequently, the $S_n$ orbit $O_{S_n}((I_\pi, J_\pi))$ does not split.
\end{proof}

\begin{remark}
	The case for $t \leq n - 2$ in Theorem \ref{SnOrbitSplit} was proven independently in \cite{maron19a}.
\end{remark}

As a result, we immediately obtain the following two theorems: 
\begin{theorem} \label{AnSetPart}
	Let $k, l \in \mathbb{Z}_{\geq 0}$ and $n \in \mathbb{Z}_{\geq 1}$.

	If $\pi$ is a set partition in $\Pi_{l+k,n}$ having $n-2$ blocks or fewer, then $\pi$ corresponds bijectively to a basis element of 
	$\Hom_{A_n}((\mathbb{R}^{n})^{\otimes k}, (\mathbb{R}^{n})^{\otimes l})$,

	Otherwise, $\pi$ corresponds to two basis elements of 
	$\Hom_{A_n}((\mathbb{R}^{n})^{\otimes k}, (\mathbb{R}^{n})^{\otimes l})$.
\end{theorem}
\begin{theorem} \label{Andimension}
	For $k, l \in \mathbb{Z}_{\geq 0}, n \in \mathbb{Z}_{\geq 1}$, 
	the dimension of 
	$
	\Hom_{A_n}((\mathbb{R}^{n})^{\otimes k}, (\mathbb{R}^{n})^{\otimes l})
	$
is equal to
\begin{equation}
	\sum_{t=1}^{n-2} 
	\begin{Bmatrix}
		l+k\\
		t 
	\end{Bmatrix}
	+
	2
	\begin{Bmatrix}
		l+k\\
		n-1
	\end{Bmatrix}
	+
	2
	\begin{Bmatrix}
		l+k\\
		n 
	\end{Bmatrix}
\end{equation}
\end{theorem}
In other words, Theorem \ref{AnSetPart} tells us that finding a basis of
	$
	\Hom_{A_n}((\mathbb{R}^{n})^{\otimes k}, (\mathbb{R}^{n})^{\otimes l})
	$
in the standard basis of $\mathbb{R}^{n}$ is very similar to finding a basis of 
	$
	\Hom_{S_n}((\mathbb{R}^{n})^{\otimes k}, (\mathbb{R}^{n})^{\otimes l})
	$
in the standard basis of $\mathbb{R}^{n}$, since finding the basis amounts once again to considering all of the set partitions of $[l+k]$ having at most $n$ blocks.

Indeed, if a set partition $\pi$ in $\Pi_{l+k,n}$ has at most $n-2$ blocks, then we see that $X_\pi$ given in (\ref{equivbasiselement}) is a basis element of 
	$
	\Hom_{A_n}((\mathbb{R}^{n})^{\otimes k}, (\mathbb{R}^{n})^{\otimes l})
	$
since, in this case, $O_{S_n}((I_\pi, J_\pi)) = O_{A_n}((I_\pi, J_\pi))$, 
%as the $S_n$ orbit on $[n]^{l+k}$ does not split, 
by Theorem \ref{SnOrbitSplit}.

The question remains as to how to take a set partition $\pi$ in $\Pi_{l+k,n}$ having either $n-1$ or $n$ blocks and use it to obtain the two basis elements of 
	$
	\Hom_{A_n}((\mathbb{R}^{n})^{\otimes k}, (\mathbb{R}^{n})^{\otimes l})
	$ 
that it corresponds to.
Said differently, we 
%need to be able to 
would like to
take such a set partition $\pi$ and identify the two $A_n$ orbits that its 
corresponding
$S_n$ orbit on $[n]^{l+k}$,
$O_{S_n}((I_\pi, J_\pi))$, splits into.

\subsection{Enter: the Jellyfish}

%\textbf{NEED TO SAY HOW MY EXPOSITION AND PROOFS ARE DIFFERENT FROM THOSE GIVEN THERE}
%The idea to use jellyfish to identify the two $A_n$ orbits is due to 
%%The following idea is motivated by the paper written by 
%Comes~\yrcite{comes}; here, we simplify and add to the arguments given in their paper to suit our needs.

Comes~\yrcite{comes} originally came up with the idea of using \textit{jellyfish} to identify the two $A_n$ orbits; here, however, our choice of exposition is rather different and our proofs are simpler than those given in their paper.

We begin by defining the following map.

\begin{definition}
For $n \in \mathbb{Z}_{\geq 1}$, 
	we define the determinant map
\begin{equation} \label{detmap}
	\det : (\mathbb{R}^{n})^{\otimes n} \rightarrow \mathbb{R}
\end{equation}
	on the standard basis of $(\mathbb{R}^{n})^{\otimes n}$ by
\begin{equation} \label{detmapdefn}
	e_{I} = e_{i_1} \otimes \dots \otimes e_{i_n}
	\mapsto
	\begin{vmatrix}
		e_{i_1} & \dots & e_{i_n}
	\end{vmatrix}
\end{equation}
and extend linearly, 
	where the RHS of (\ref{detmapdefn}) is the determinant of an $n \times n$ matrix.
\end{definition}

\begin{lemma}
The determinant map is an element of 
	$\Hom_{A_n}((\mathbb{R}^{n})^{\otimes n}, (\mathbb{R}^{n})^{\otimes 0})$, 
but it is not an element of
	$\Hom_{S_n}((\mathbb{R}^{n})^{\otimes n}, (\mathbb{R}^{n})^{\otimes 0})$.
\end{lemma}

\begin{proof}
	It is clear that the determinant map is a linear map.

	%It can be shown that, for any $\sigma \in S_n$,
	Then, for any $\sigma \in S_n$, we have that
	\begin{equation} \label{detSnequiv}
		\det(\rho_n(\sigma)[e_{I}]) 
		= 
		(-1)^{\sgn(\sigma)} \det(e_{I})
	\end{equation}
	since a transposition in $S_n$ corresponds to swapping two columns of the $n \times n$ matrix.

	As $\sgn(\sigma) = 1$ for all $\sigma \in A_n$, 
	(\ref{detSnequiv}) shows that
$\det$ is an element of 
	$\Hom_{A_n}((\mathbb{R}^{n})^{\otimes n}, (\mathbb{R}^{n})^{\otimes 0})$.
	It is enough to see that for any odd permutation in $S_n$, (\ref{detSnequiv}) implies that
$\det$ is not an element of 
	$\Hom_{S_n}((\mathbb{R}^{n})^{\otimes n}, (\mathbb{R}^{n})^{\otimes 0})$.
\end{proof}

For each $n \in \mathbb{Z}_{\geq 1}$, 
we represent the determinant map by a diagram that has a single row of $n$ vertices, each of which is attached to a blue head.
We will call this diagram a \textit{jellyfish}, for obvious reasons.

For example, if $n = 3$, the jellyfish has the form
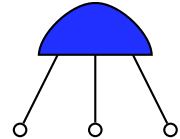
\begin{center}
	\scalebox{0.5}{\input{jellyfishex.tikz}}
\end{center}
\begin{remark}
	It is important to highlight that the determinant map (\ref{detmap}) is different from the determinant operator on an $n \times n$ matrix.
	In particular, the determinant map is a linear map on 
	$(\mathbb{R}^{n})^{\otimes n}$
	whereas the determinant operator on $n \times n$ matrices is not linear.
\end{remark}

We introduce the following useful lemma.

\begin{lemma} \label{detbasis}
	The determinant map applied to a standard basis vector of 
	$(\mathbb{R}^{n})^{\otimes n}$
	gives either $-1, 0$ or $+1$.
\end{lemma}

\begin{proof}
	Let
	$e_{I} = e_{i_1} \otimes \dots \otimes e_{i_n}$
	be a standard basis vector of 
	$(\mathbb{R}^{n})^{\otimes n}$.

	If $i_x = i_y$ for some $1 \leq x \neq y \leq n$, then $\det(e_I) = 0$ since two columns in the $n \times n$ matrix will be the same.
	Otherwise, $\{i_1, \dots, i_n\}$ is some permutation of $[n]$.
	Defining
	$e_{[n]} \coloneqq e_{1} \otimes \dots \otimes e_{n}$, we know that
	$\det(e_{[n]}) = +1$.
	Since, in this case, $e_{I} = \rho_n(\sigma)e_{[n]}$ for some $\sigma \in S_n$, 
	 we can use (\ref{detSnequiv}) to calculate 
	$\det(e_{I})$, giving a result of $\pm 1$.
\end{proof}

\begin{remark} \label{remsplit}
%Another way to view
We can view
Lemma \ref{detbasis} in another way; namely, that the determinant map 
%which is $A_n$-equivariant but not $S_n$-equivariant, 
	splits basis vectors of
$(\mathbb{R}^{n})^{\otimes n}$ 
into disjoint classes.
Let us call these classes $-1, 0$ and $+1$.
\end{remark}

Consequently, Remark \ref{remsplit}
suggests that the determinant map is a possible candidate function
for identifying the $A_n$ orbits that the $S_n$ orbit
$O_{S_n}((I_\pi, J_\pi))$ splits into, where 
%the $S_n$ orbit corresponds to a set partition 
$\pi$ is a set partition in $\Pi_{l+k,n}$ having either $n-1$ or $n$ blocks.

%\textbf{Rough idea for how to proceed:}
%\textbf{CHANGE THIS SENTENCE}
%Lemma \ref{detbasis} gives us the key insight for how to proceed:
%for a given set partition $\pi$ in $\Pi_{l+k,n}$ having either $n-1$ or $n$ blocks, we want to use  
%the determinant map to identify which $A_n$ orbit the 
%elements of our $S_n$ orbit belong to, since we know that the $S_n$ orbit splits.

%The way that we will achieve this is by using the determinant map to map the elements of one $A_n$ orbit to $+1$ and the other to $-1$.

%We take the set partition $\pi$ in $\Pi_{l+k,n}$ having either $n-1$ or $n$ blocks and its corresponding
%$S_n$ orbit $O_{S_n}((I_\pi, J_\pi))$, which we know splits into $2$ $A_n$ orbits.

However, since the determinant map is a map
	$(\mathbb{R}^{n})^{\otimes n} \rightarrow \mathbb{R}$,
%we need to somehow attach its corresponding diagram to our set partition diagram.
%Ideally, 
	and the elements of $O_{S_n}((I_\pi, J_\pi))$ are elements of $[n]^{l+k}$,
	to use the determinant map to try to identify the $A_n$ orbits in $O_{S_n}((I_\pi, J_\pi))$,
	%we need to create a map
%we want to create a map
	it would be useful to create a map
	$g_\pi : (\mathbb{R}^{n})^{\otimes l+k} \rightarrow (\mathbb{R}^{n})^{\otimes n}$
that corresponds bijectively with the set partition $\pi$,
%which can then be composed with the determinant map to create a function
%	$f_\pi \coloneqq \det \circ \, g_\pi : (\mathbb{R}^{n})^{\otimes l+k} \rightarrow \mathbb{R}$.
%Clearly, $f_\pi$ would also correspond bijectively with the set partition $\pi$.
since such a map would project standard basis elements of $(\mathbb{R}^{n})^{\otimes l+k}$ onto, at the very least, a linear combination of basis elements of $(\mathbb{R}^{n})^{\otimes n}$.
We could then use the splitting property of the determinant map on such linear combinations to try to split the elements of $O_{S_n}((I_\pi, J_\pi))$ into different classes.
However, in order to use the three classes given in Remark \ref{remsplit}, we cannot choose any map $(\mathbb{R}^{n})^{\otimes l+k} \rightarrow (\mathbb{R}^{n})^{\otimes n}$.
We will see that with a clever choice of $g_\pi$, the determinant map then splits the elements of $O_{S_n}((I_\pi, J_\pi))$ into the $\pm 1$ classes and sends all other elements of $[n]^{l+k}$ to the $0$ class.

%\textbf{IT FEELS LIKE WE'RE MISSING A BIT HERE THAT WOULD GIVE US EXTRA INSIGHT UP FRONT AS TO WHY WHAT'S ABOUT TO FOLLOW WILL WORK ...}

%\textbf{I THINK THE ROUGH IDEA IS THIS: WE USE THE MAP $g_\pi$ TO PROJECT BASIS ELEMENTS OF $(\mathbb{R}^{n})^{\otimes l+k}$ ONTO A LINEAR COMBINATION OF BASIS ELEMENTS OF $(\mathbb{R}^{n})^{\otimes n}$; FOR THOSE BASIS ELEMENTS OF $(\mathbb{R}^{n})^{\otimes l+k}$ INDEXED BY THE ELEMENTS OF $O_{S_n}((I_\pi, J_\pi))$, THE DETERMINANT MAP WILL DO THE APPROPRIATE SPLITTING, AND FOR THE REST, IT WILL DISCARD THEM (SEND TO $0$)}

%since their composition
%	$f_\pi \coloneqq \det \circ \, g_\pi : (\mathbb{R}^{n})^{\otimes l+k} \rightarrow \mathbb{R}$
%would also correspond bijectively with the set partition $\pi$.
%\textbf{WHY DO WE WANT TO DO THIS?}

%\textbf{The map $g_\pi$ is $S_n$ equivariant since it corresponds to a set partition in $\Pi_{n+l+k,n}$; in fact, it has exactly $n$ blocks!}

%The way to do this is 
We create $g_\pi$ as follows.
Flatten the diagram $d_{\pi}$ that corresponds to $\pi$.
Add a new top row of $n$ vertices above the row of $l+k$ vertices.
From the bottom row, take the lowest numbered vertex in block $i$ of $\pi$, $1 \leq i \leq t$, where $t = n - 1$ or $n$, and connect that vertex to vertex $i$ in the top row.
Call this new diagram $b_\pi$.

%\textbf{ADD AN EXAMPLE HERE}

We claim the following:

\begin{proposition} \label{gpiSnequiv}
	The diagram $b_\pi$ corresponds bijectively with a basis element of
	$\Hom_{S_n}((\mathbb{R}^{n})^{\otimes l+k}, (\mathbb{R}^{n})^{\otimes n})$.
	Consequently, we define $g_\pi$ to be this basis element.
\end{proposition}

\begin{proof}
	It is clear that the diagram $b_\pi$ corresponds to a set partition of $[n+l+k]$ having exactly $n$ blocks. 
	Hence this set partition is an element of $\Pi_{n+l+k,n}$.
	Since each set partition in $\Pi_{n+l+k,n}$ corresponds bijectively with a basis element of  
	$\Hom_{S_n}((\mathbb{R}^{n})^{\otimes l+k}, (\mathbb{R}^{n})^{\otimes n})$,
	by Theorem \ref{equivbasislplusk}, we obtain the result.
\end{proof}

The following result is immediate from the construction of the diagram $b_\pi$.
\begin{proposition}
	In the standard basis of matrix units of
	$\Hom((\mathbb{R}^{n})^{\otimes l+k}, (\mathbb{R}^{n})^{\otimes n})$,
	$g_\pi$ is given by
\begin{equation} \label{gpiform}
	g_\pi \coloneqq 
	\sum_{(K,I,J) \in O_{S_n}((K_{\pi},I_{\pi},J_{\pi}))} E_{K,(I,J)}
\end{equation}
where
\begin{equation} 
	K_{\pi} \coloneqq (1, 2, \dots, n)
\end{equation}
and $O_{S_n}((K_{\pi},I_{\pi},J_{\pi}))$ is defined to be
\begin{equation} \label{OKpiIpiJpi}
	\left\{ (K,I,J) \; \middle\vert 
	\begin{array}{l}
		(I,J) \in O_{S_n}((I_{\pi},J_{\pi})) \\
		\text{and if } (I,J) = \sigma(I_{\pi},J_{\pi}), \\
		\text{then } K \coloneqq \sigma(K_{\pi})
	\end{array}
	\right\}
\end{equation}
\end{proposition}

In order to see that $g_\pi$ has the properties that we need, we first define the following map.
\begin{definition}
	Let $\pi$ be a set partition in $\Pi_{l+k,n}$ having either $n-1$ or $n$ blocks.
	Then we define $f_\pi \coloneqq \det \circ \, g_\pi : (\mathbb{R}^{n})^{\otimes l+k} \rightarrow \mathbb{R}$.
	Clearly, $f_\pi$ corresponds bijectively with the set partition $\pi$.
\end{definition}

%For each set partition $\pi \in \Pi_{l+k, n}$ having either $n-1$ or $n$ blocks, 
We can associate to $f_\pi$ a diagram that is built from the composition of $b_\pi$ (which corresponds to the map $g_\pi$) and a jellyfish (which corresponds to the determinant map).

For example, the diagram that is associated to $f_\pi$ for the set partition $\pi$ whose diagram $d_\pi$ is given in Figure \ref{partklelement1},
where we choose $l=3$, $k=5$ and $n=5$, is
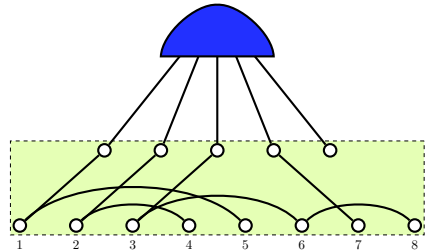
\begin{center}
	\scalebox{0.5}{\input{jellyfishattached.tikz}}
\end{center}

\begin{proposition} \label{fpiAnequiv}
	The map $f_\pi$ is an element of
	$\Hom_{A_n}((\mathbb{R}^{n})^{\otimes l+k}, (\mathbb{R}^{n})^{\otimes 0})$.
\end{proposition}

\begin{proof}
	Let $e_{(I,J)}$ be any standard basis vector in $(\mathbb{R}^{n})^{\otimes l+k}$, where $I \in [n]^l$ and $J \in [n]^k$.

	Then, for all $\sigma \in A_n$, we have that
	\begin{align}
		f_\pi(\rho_{l+k}(\sigma)[e_{(I,J)}])
		& =
		\det \circ \, g_\pi 
		(\rho_{l+k}(\sigma)[e_{(I,J)}]) \\
		& =
		\det(\rho_n(\sigma)[g_\pi(e_{(I,J)})]) \label{gpiequiv} \\
		& =
		\det(g_\pi(e_{(I,J)})) \label{gpidet} \\
		& =
		f_\pi(e_{(I,J)})
	\end{align}
	where, in (\ref{gpiequiv}), we have used Proposition \ref{gpiSnequiv}
	and in (\ref{gpidet}), we have used (\ref{detSnequiv}).
\end{proof}

We come to the crucial result for our purposes.
\begin{theorem} \label{fpimaps}
	Let $e_{(I,J)}$ be any standard basis vector in $(\mathbb{R}^{n})^{\otimes l+k}$, where $I \in [n]^l, J \in [n]^k$.

	Then
	\begin{equation}
		f_\pi : e_{(I,J)} \rightarrow
		\begin{cases}
			\pm 1 & \text{if $(I,J) \in O_{S_n}((I_\pi, J_\pi))$} \\
			0     & \text{otherwise}
		\end{cases}
	\end{equation}
\end{theorem}

\begin{proof}
	For any standard basis vector $e_{(I,J)}$ in 
	$(\mathbb{R}^{n})^{\otimes l+k}$, 
	we have, on a matrix unit of
	$\Hom((\mathbb{R}^{n})^{\otimes l+k}, (\mathbb{R}^{n})^{\otimes n})$, that
	\begin{equation} \label{matrixunitmult}
		E_{K, (L,M)}e_{(I,J)}
		=
		\delta_{((L,M),(I,J))}e_{K}
	\end{equation}
	where $e_{K}$ is a standard basis vector in 
	$(\mathbb{R}^{n})^{\otimes n}$.
	
	Hence, by (\ref{gpiform}), we see that,
	for the otherwise case,
	%if $(I,J) \notin O_{S_n}((I_\pi, J_\pi))$, then
	$g_\pi(e_{(I,J)}) = 0$, and so 
	$f_\pi(e_{(I,J)}) = 0$. 

	If $(I,J) \in O_{S_n}((I_\pi, J_\pi))$, then we have that 
	$(I,J) = \sigma(I_{\pi},J_{\pi})$ for some $\sigma \in S_n$. 

	Hence, by (\ref{gpiform}), (\ref{matrixunitmult}) and
	the definition of
	$O_{S_n}((K_{\pi},I_{\pi},J_{\pi}))$
	given in (\ref{OKpiIpiJpi}), 
	we see that 
	$g_\pi(e_{(I,J)}) = e_{K}$,
	where	
	$K = \sigma(K_{\pi})$.

	Since $K$ is a permutation of $[n]$, we can apply Lemma \ref{detbasis}
	to get that
	$f_\pi(e_{(I,J)}) = \det(e_{K}) = \pm 1$. 
\end{proof}

Consequently, we can define the following two sets:
\begin{equation} \label{Opiplus}
	O_\pi^{+} 
	\coloneqq
	\{(I,J) \in O_{S_n}((I_\pi, J_\pi)) \mid f_\pi(e_{(I,J)}) = +1\}
\end{equation}
and
\begin{equation} \label{Opiminus}
	O_\pi^{-} 
	\coloneqq
	\{(I,J) \in O_{S_n}((I_\pi, J_\pi)) \mid f_\pi(e_{(I,J)}) = -1\}
\end{equation}
We claim the following result.
\begin{theorem}
	$O_\pi^{+}$ and $O_\pi^{-}$ are the two $A_n$ orbits that $O_{S_n}((I_\pi, J_\pi))$ splits into.
\end{theorem}
\begin{proof}
	It is enough to show that if $(I,J) \in O_\pi^{+}$, then $\sigma(I,J) \in O_\pi^{+}$ for all $\sigma \in A_n$, and
	if $(I,J) \in O_\pi^{-}$, then $\sigma(I,J) \in O_\pi^{-}$ for all $\sigma \in A_n$,
	where the action of $\sigma$ on the pair $(I,J)$ is given in (\ref{sigmapairaction}).
	Indeed, by Proposition \ref{fpiAnequiv}, we have that
	\begin{equation}
		f_\pi(e_{\sigma(I,J)})
		=
		f_\pi(\rho_{l+k}(\sigma)[e_{(I,J)}])
		= 
		f_\pi(e_{(I,J)})
	\end{equation}
	Applying Theorem \ref{fpimaps} gives the result.
\end{proof}

Relabelling 
$O_\pi^{+}$ as $O_{A_n}((I_\pi^{+}, J_\pi^{+}))$, and
$O_\pi^{-}$ as $O_{A_n}((I_\pi^{-}, J_\pi^{-}))$, we obtain the two basis elements of 
	$\Hom_{A_n}((\mathbb{R}^{n})^{\otimes k},(\mathbb{R}^{n})^{\otimes l})$ 
from the one set partition $\pi \in \Pi_{l+k,n}$ that has either $n-1$ or $n$ blocks,
	namely
\begin{equation} \label{Xpiplus}
	X_\pi^{+} \coloneqq \sum_{(I,J) \in O_{A_n}((I_\pi^{+}, J_\pi^{+}))} E_{I,J}
\end{equation}
and
\begin{equation} \label{Xpiminus}
	X_\pi^{-} \coloneqq \sum_{(I,J) \in O_{A_n}((I_\pi^{-}, J_\pi^{-}))} E_{I,J}
\end{equation}

We summarise the above results into the following theorem.
\begin{theorem} \label{mainAnthm}
	For $l, k \in \mathbb{Z}_{\geq 0}, n \in \mathbb{Z}_{\geq 1}$,
	the union of the two sets
	\begin{equation}
		\{X_\pi \mid \pi \in \Pi_{l+k,n-2}\}
	\end{equation}
	\begin{equation}
		\{X_\pi^{+}, X_\pi^{-} \mid \pi \in \Pi_{l+k,n}\setminus\Pi_{l+k,n-2}\}
	\end{equation}
	forms a basis of 
	$\Hom_{A_n}((\mathbb{R}^{n})^{\otimes k}, (\mathbb{R}^{n})^{\otimes l})$.
	Its dimension is given in Theorem \ref{Andimension}.
\end{theorem}

In Algorithm \ref{alg1},
we present some pseudocode 
for how to explicitly construct the 
weight matrix for an $A_n$-equivariant linear layer mapping
$(\mathbb{R}^{n})^{\otimes k} \rightarrow (\mathbb{R}^{n})^{\otimes l}$ in the standard basis of $\mathbb{R}^{n}$.
We assume that we have access to the following procedures:
\begin{itemize}
	\item \textsc{SymmGrpOrbit} calculates the $S_n$ orbit for a set partition $\pi \in \Pi_{l+k,n}$.
	\item \textsc{Flatten} takes a set partition diagram with $l$ vertices in the top row and $k$ vertices in the bottom row, and returns its equivalent flattened set partition diagram having a single row of $l+k$ vertices, as per Section \ref{setpartitions}.
	\item \textsc{NewTopRowConnect} takes a flattened set partition diagram having either $n-1$ or $n$ blocks, inserts a new top row consisting of $n$ vertices, and connects the lowest numbered vertex in each block $i$ to vertex $i$ in the top row.
	\item \textsc{AttachJellyfish} takes a set partition diagram $b_\pi$ with $n$ vertices in the top row and $l+k$ vertices in the bottom row having either $n-1$ or $n$ blocks, attaches a jellyfish with $n$ legs to the top row, and returns the function $f_\pi$ that is associated with this new diagram.
\end{itemize}

We give an example that explicitly shows how to construct such a weight matrix in the Technical Appendix, in the case where $n=2$, $k=2$ and $l=1$.

We appreciate that there will be some technical challenges 
when implementing Algorithm \ref{alg1} given the current state of computer hardware. 
We discuss this in more detail in the Technical Appendix.

It is possible to extend our results by looking at linear layer functions that are equivariant to a direct product of alternating groups; 
this is given in full in
%full details are given in 
the Technical Appendix.

\begin{algorithm}[tb]
	\caption{How to Calculate the Weight Matrix for an $A_n$ Equivariant Linear Layer Mapping $(\mathbb{R}^{n})^{\otimes k} \rightarrow (\mathbb{R}^{n})^{\otimes l}$}
	\label{alg1}
\begin{algorithmic}
   \STATE {\bfseries Input:} $n$, $k$, $l$
   \STATE {\bfseries Output:} Weight Matrix $M$
	%data $x_i$, size $m$
   \STATE Initialize $M = 0$.
   \FORALL{$\pi \in \Pi_{l+k,n-2}$}
	%\STATE $(I_\pi, J_\pi) = $ \textsc{BlockLabel}$(\pi)$
	%\STATE $O_{S_n}((I_\pi, J_\pi)) = $ \textsc{Orbit}$((I_\pi, J_\pi))$
	\STATE $O_{S_n}((I_\pi, J_\pi)) = $ \textsc{SymmGrpOrbit}$(\pi, n)$
	\STATE $X_\pi = \sum_{(I,J) \in O_{S_n}((I_{\pi},J_{\pi}))} E_{I,J}$
	\STATE $M = M + \lambda_\pi X_\pi$
   \ENDFOR
   \FORALL{$\pi \in \Pi_{l+k,n} \setminus \Pi_{l+k,n-2}$}
	\STATE $d_\pi = $ \textsc{Flatten}$(d_\pi)$
	\STATE $b_\pi = $ \textsc{NewTopRowConnect}$(d_\pi, n)$
	\STATE $f_\pi = $ \textsc{AttachJellyfish}$(b_\pi, n)$
	%\STATE $(I_\pi, J_\pi) = $ \textsc{BlockLabel}$(\pi)$
	%\STATE $O_{S_n}((I_\pi, J_\pi)) = $ \textsc{Orbit}$((I_\pi, J_\pi))$
	\STATE $O_{S_n}((I_\pi, J_\pi)) = $ \textsc{SymmGrpOrbit}$(\pi, n)$
	\FORALL{$(I,J) \in O_{S_n}((I_\pi, J_\pi))$}
		\IF{$f_\pi(e_{(I,J)}) = +1$}
			\STATE $O_{A_n}((I_\pi^{+}, J_\pi^{+})).$\textsc{Append}$((I,J))$
		\ELSE
			\STATE $O_{A_n}((I_\pi^{-}, J_\pi^{-})).$\textsc{Append}$((I,J))$
		\ENDIF
   	\ENDFOR
	\STATE $X_\pi^{+} = \sum_{(I,J) \in O_{A_n}((I_\pi^{+}, J_\pi^{+}))} E_{I,J}$
	\STATE $X_\pi^{-} = \sum_{(I,J) \in O_{A_n}((I_\pi^{-}, J_\pi^{-}))
} E_{I,J}$
	\STATE $M = M + \lambda_\pi^{+} X_\pi^{+} + \lambda_\pi^{-} X_\pi^{-}$
   \ENDFOR
   \STATE {\bfseries Return:} $M$
\end{algorithmic}
\end{algorithm}

\section{Adding Features and Biases}

\subsection{Features} 

We have assumed throughout that the feature dimension for all of the layers appearing in the neural network is one.
We can adapt all of the results that have been shown for the case where the feature dimension of the layers is greater than one.

Suppose that an $r$-order tensor has a feature space of dimension $d_r$.
We now wish to find a basis for 
\begin{equation} \label{Homfeatures}
	\Hom_{A_n}((\mathbb{R}^{n})^{\otimes k} \otimes \mathbb{R}^{d_k}, (\mathbb{R}^{n})^{\otimes l} \otimes \mathbb{R}^{d_l})
\end{equation}
in the standard basis of $\mathbb{R}^{n}$. 
	
Such a basis can be found by making the following substitutions,
where now $i \in [d_l]$ and $j \in [d_k]$:
\begin{itemize}
	\item replace $E_{I,J}$ by $E_{I,i,J,j}$ in (\ref{equivbasiselement}), (\ref{Xpiplus}), and (\ref{Xpiminus})
	\item relabel $X_\pi$ by $X_{\pi, i, j}$, 
		$X_\pi^{+}$ by $X_{\pi, i, j}^{+}$, and
		$X_\pi^{-}$ by $X_{\pi, i, j}^{-}$.
\end{itemize}
Consequently, a basis for (\ref{Homfeatures}) in the standard basis of $\mathbb{R}^{n}$ is given by the union of the two sets
	\begin{equation}
		\{X_{\pi, i, j} \mid \pi \in \Pi_{l+k,n-2}, i \in [d_l], j \in [d_k]\}
	\end{equation}
	\begin{equation}
		\left\{X_{\pi, i, j}^{+}, X_{\pi, i, j}^{-} \;\middle\vert 
		\begin{array}{l}
		\pi \in \Pi_{l+k,n}\setminus\Pi_{l+k,n-2}, \\
		i \in [d_l], j \in [d_k]
		\end{array}
		\right\}
	\end{equation}

\subsection{Biases} 
Including bias terms in the layer functions of a $A_n$-equivariant neural network is harder, but it can be done.
For the learnable linear layers of the form
$\Hom_{A_n}((\mathbb{R}^{n})^{\otimes k}, (\mathbb{R}^{n})^{\otimes l})$,
Pearce--Crump~\yrcite{pearcecrump} shows that the $A_n$-equivariance of the bias function, 
$\beta : ((\mathbb{R}^{n})^{\otimes k}, \rho_k) \rightarrow 
((\mathbb{R}^{n})^{\otimes l}, \rho_l)$,
needs to satisfy
	\begin{equation} \label{Anbiasequiv}
		c = \rho_l(g)c
	\end{equation}
for all $g \in A_n$ and $c \in (\mathbb{R}^{n})^{\otimes l}$.

Since any $c \in (\mathbb{R}^{n})^{\otimes l}$
satisfying (\ref{Anbiasequiv}) can be viewed as an element of 
$\Hom_{A_n}(\mathbb{R}, (\mathbb{R}^{n})^{\otimes l})$,
to find the matrix form of $c$, all we need to do is to find a basis for 
$\Hom_{A_n}(\mathbb{R}, (\mathbb{R}^{n})^{\otimes l})$.

But this is simply a matter of applying Theorem \ref{mainAnthm},
setting $k = 0$.

\section{Related Work}

The theory for the alternating group has its roots in the theory for the symmetric group and its links to the partition algebra.
Jones~\yrcite{Jones} constructed a surjective algebra homomorphism between the partition algebra $P_k(n)$ and the centraliser algebra of the symmetric group, $\End_{S_n}((\mathbb{R}^{n})^{\otimes k})$.
Most notably, Benkart and Halverson went on to develop much of the theory for the duality between the symmetric group and the partition algebra in a number of important papers~\yrcite{BHH, BenHal1, BenHal2}.
Bloss~\yrcite{bloss} used the result of Jones~\yrcite{Jones} to study the centralizer algebra of the alternating group, $\End_{A_n}((\mathbb{R}^{n})^{\otimes k})$.
He showed that the partition algebra $P_k(n)$, which has a basis consisting of set partition diagrams having two rows of $k$ vertices, is isomorphic to the centralizer algebra when $n \geq 2k + 2$.
He also highlighted the difficulty of finding a diagrammatic approach for characterising $\End_{A_n}((\mathbb{R}^{n})^{\otimes k})$ in the remaining cases since he recognised that the $S_n$ orbits that correspond bijectively with the set partition diagrams split in these cases.
Comes~\yrcite{comes} solved this problem and extended it to all $\Hom$--spaces
$\Hom_{A_n}((\mathbb{R}^{n})^{\otimes k}, (\mathbb{R}^{n})^{\otimes l})$, developing much of the theory in the process.
In particular, he came up with the idea of using the 
determinant map to show how the $S_n$ orbits split, and introduced jellyfish to represent this map in its diagrammatic form.

Finding the form of neural networks that are equivariant to a particular group has become an important area of research.
%field of study ever since
%Group equivariant neural networks, on the other hand, have been a hot topic of study in the artificial intelligence community ever since
Zaheer et al.~\yrcite{deepsets} introduced the first permutation equivariant neural network, called Deep Sets, for learning from sets in a permutation equivariant manner. 
Maron et al.~\yrcite{maron2018} were the first to study the problem of classifying all of the linear permutation equivariant and invariant neural network layers, with their motivation coming from learning relations between the nodes of graphs. 
They characterised all of the learnable, linear, permutation equivariant layer functions in 
$\Hom_{S_n}((\mathbb{R}^{n})^{\otimes k}, (\mathbb{R}^{n})^{\otimes l})$
in the practical cases (specifically, when $n \geq k+l$). 
They used some equalities involving Kronecker products to find a number of fixed point equations which they solved to find a basis, in tensor form, for the layer functions under consideration.

As discussed in the Introduction, 
the approach taken in this paper to
characterise all of the learnable, linear, equivariant layer functions in 
$\Hom_{A_n}((\mathbb{R}^{n})^{\otimes k}, (\mathbb{R}^{n})^{\otimes l})$
is similar to the one seen in the papers written by
Pearce--Crump~\yrcite{pearcecrump, pearcecrumpB}.
They used various sets of set partition diagrams to characterise
all of the learnable, linear, equivariant layer functions in 
$\Hom_{G}((\mathbb{R}^{n})^{\otimes k}, (\mathbb{R}^{n})^{\otimes l})$
when $G$ is any of the following groups:
the symmetric group $S_n$, the orthogonal group $O(n)$, the symplectic group $Sp(n)$, and the special orthogonal group $SO(n)$.
%We suggest that this technique may be useful to find group equivariant neural network architectures for other groups.

\section{Conclusion}

We are the first to show how the combinatorics underlying set partition diagrams, together with some jellyfish representing the determinant map, provides the theoretical background for constructing neural networks that are equivariant to the alternating group when the layers are some tensor power of $\mathbb{R}^{n}$.
We looked at the problem of calculating the form of the learnable, linear, $A_n$--equivariant layer functions between such tensor power spaces in the standard basis of $\mathbb{R}^{n}$.
We achieved this by finding a basis for the $\Hom$--spaces in which these layer functions live.
In particular, we showed how set partition diagrams correspond to a number of basis elements, and we used jellyfish to identify the individual basis elements when a set partition diagram corresponded to more than one basis element.
In doing so, we calculated the number of parameters that appear in these layer functions. 
We also generalised our approach to show how to construct neural networks that are equivariant to local symmetries.

\section*{Acknowledgements}

The author would like to thank his PhD supervisor Professor William J. Knottenbelt for being generous with his time throughout the author's period of research prior to the publication of this paper.

This work was funded by the Doctoral Scholarship for Applied Research which was awarded to the author under Imperial College London's Department of Computing Applied Research scheme.
This work will form part of the author's PhD thesis at Imperial College London.

\nocite{*}
\bibliography{references}
\bibliographystyle{icml2023}

%%%%%%%%%%%%%%%%%%%%%%%%%%%%%%%%%%%%%%%%%%%%%%%%%%%%%%%%%%%%%%%%%%%%%%%%%%%%%%%
%%%%%%%%%%%%%%%%%%%%%%%%%%%%%%%%%%%%%%%%%%%%%%%%%%%%%%%%%%%%%%%%%%%%%%%%%%%%%%%
% APPENDIX
%%%%%%%%%%%%%%%%%%%%%%%%%%%%%%%%%%%%%%%%%%%%%%%%%%%%%%%%%%%%%%%%%%%%%%%%%%%%%%%
%%%%%%%%%%%%%%%%%%%%%%%%%%%%%%%%%%%%%%%%%%%%%%%%%%%%%%%%%%%%%%%%%%%%%%%%%%%%%%%
\newpage
\appendix
\onecolumn

\section{The Weight Matrix for $S_2$ and $A_2$--Equivariant Linear Layer Functions from $(\mathbb{R}^{2})^{\otimes 2}$ to $\mathbb{R}^{2}$}

	We first show how to find a basis of 
	$\Hom_{S_2}((\mathbb{R}^{2})^{\otimes 2}, \mathbb{R}^{2})$, 
	where, in this case, $n=2$, $k=2$, and $l=1$.

	To do this, we first need to find all set partitions in $\Pi_{1+2,2}$, that is, all set partitions $\pi$ of $3$ having at most $2$ blocks.
	The first column of Figure \ref{matrix2,2^1} shows these set partitions in their equivalent diagram form, $d_\pi$.
	Each of these set partition diagrams corresponds to a basis element in 
	$\Hom_{S_2}((\mathbb{R}^{2})^{\otimes 2}, \mathbb{R}^{2})$, by Theorem \ref{Sncorresp}.

	To obtain the basis elements themselves, we first recall that each set partition $\pi$ corresponds to an orbit 
	$O_{S_2}((I_\pi, J_\pi))$
	coming from the action of $S_2$ on $[2]^{1+2}$.
	A representative of each orbit, called the block labelling, is given in the third column of Figure \ref{matrix2,2^1},
	and it is found by letting the $i^{\text{th}}$ position in $[1+2]$ be the label of the block in $\pi$ containing the number $i$, where the blocks of $\pi$ have themselves been labelled by letting
$B_1$ be the block that contains the number $1 \in [1+2]$, and, if $\pi$ has $2$ blocks, letting $B_2$ be the other block in $\pi$.
%iteratively letting $B_j$, for $1 < j \leq t$, be the block that contains the smallest number in $[l+k]$ that is not in $B_1 \cup B_2 \cup \dots \cup B_{j-1}$.

	Finally, we form a basis element $X_\pi$
	of 
	$\Hom_{S_2}((\mathbb{R}^{2})^{\otimes 2}, \mathbb{R}^{2})$ 
	by summing over all matrix units in
	$\Hom((\mathbb{R}^{2})^{\otimes 2}, \mathbb{R}^{2})$
	that are indexed by the elements of
	$O_{S_2}((I_\pi, J_\pi))$, as stated in (\ref{equivbasiselement}).
	These basis elements are given in the fourth column of Figure \ref{matrix2,2^1}.

	\begin{figure}[h]
\begin{center}
\begin{tblr}{
  colspec = {X[c,h]X[c]X[c]X[c]},
  stretch = 0,
  rowsep = 5pt,
  hlines = {1pt},
  vlines = {1pt},
}
	{Set Partition Diagram \\ $d_\pi$} & 
	{Partition \\ $\pi$} & 
	{Block Labelling \\ $(I_{\pi} \mid J_{\pi})$}	& 
	{Standard Basis Element \\ $X_\pi$ } \\
	\scalebox{0.6}{\input{orbit21sq1.tikz}} & $\{1, 2, 3\}$ & $\{1 \mid 1, 1\}$ & 
	\scalebox{0.75}{
	$
	\NiceMatrixOptions{code-for-first-row = \scriptstyle \color{blue},
                   	   code-for-first-col = \scriptstyle \color{blue}
	}
	\begin{bNiceArray}{*{2}{c} | *{2}{c}}[first-row,first-col]
				& 1,1 		& 1,2	& 2,1	& 2,2 \\
		1		& 1	& 0	& 0	& 0	\\
		2		& 0	& 0	& 0	& 1	
	\end{bNiceArray}
	$}
	\\
	\scalebox{0.6}{\input{orbit21sq2.tikz}}	& $\{1, 2 \mid 3\}$ 	& $\{1 \mid 1, 2\}$ 	& 
	\scalebox{0.75}{
	$
	\NiceMatrixOptions{code-for-first-row = \scriptstyle \color{blue},
                   	   code-for-first-col = \scriptstyle \color{blue}
	}
	\begin{bNiceArray}{*{2}{c} | *{2}{c}}[first-row,first-col]
				& 1,1 		& 1,2	& 2,1	& 2,2 \\
		1		& 0	& 1	& 0	& 0	\\
		2		& 0	& 0	& 1	& 0	
	\end{bNiceArray}
	$}
	\\
	\scalebox{0.6}{\input{orbit21sq3.tikz}}	& $\{1, 3 \mid 2\}$ 	& $\{1 \mid 2, 1\}$ 	& 
	\scalebox{0.75}{
	$
	\NiceMatrixOptions{code-for-first-row = \scriptstyle \color{blue},
                   	   code-for-first-col = \scriptstyle \color{blue}
	}
	\begin{bNiceArray}{*{2}{c} | *{2}{c}}[first-row,first-col]
				& 1,1 		& 1,2	& 2,1	& 2,2 \\
		1		& 0	& 0	& 1	& 0	\\
		2		& 0	& 1	& 0	& 0	
	\end{bNiceArray}
	$}
	\\
	\scalebox{0.6}{\input{orbit21sq4.tikz}}	& $\{1 \mid 2, 3\}$ 	& $\{1 \mid 2, 2\}$ 	& 
	\scalebox{0.75}{
	$
	\NiceMatrixOptions{code-for-first-row = \scriptstyle \color{blue},
                   	   code-for-first-col = \scriptstyle \color{blue}
	}
	\begin{bNiceArray}{*{2}{c} | *{2}{c}}[first-row,first-col]
				& 1,1 		& 1,2	& 2,1	& 2,2 \\
		1		& 0	& 0	& 0	& 1	\\
		2		& 1	& 0	& 0	& 0	
	\end{bNiceArray}
	$}
\end{tblr}
	\caption{
		We use Theorem \ref{equivbasislplusk}
		to obtain a basis of
	$\Hom_{S_2}((\mathbb{R}^{2})^{\otimes 2}, \mathbb{R}^{2})$
	from all of the
	set partitions in $\Pi_{1+2,2}$.}
  	\label{matrix2,2^1}
	\end{center}
\end{figure}

Consequently, 
%we have that 
the weight matrix for an $S_2$--equivariant linear layer function from $(\mathbb{R}^{2})^{\otimes 2}$ to $\mathbb{R}^{2}$
is of the form
\begin{equation}
	\NiceMatrixOptions{code-for-first-row = \scriptstyle \color{blue},
                   	   code-for-first-col = \scriptstyle \color{blue}
	}
	\begin{bNiceArray}{*{2}{c} | *{2}{c}}[first-row,first-col]
				& 1,1 		& 1,2	& 2,1	& 2,2 \\
		1		& \lambda_1	& \lambda_2	& \lambda_3	& \lambda_4	\\
		2		& \lambda_4	& \lambda_3	& \lambda_2	& \lambda_1	
	\end{bNiceArray}
\end{equation}
for scalars $\lambda_1, \dots, \lambda_4 \in \mathbb{R}$.

	Next, we show how to find a basis of
	$\Hom_{A_2}((\mathbb{R}^{2})^{\otimes 2}, \mathbb{R}^{2})$. 
	Again, we consider all set partitions in $\Pi_{1+2,2}$.
	%which are shown in Figure \ref{matrix2,2^1}.
	Since $n = 2$, the $S_2$ orbit corresponding to each set partition in $\Pi_{1+2,2}$
	%Figure \ref{matrix2,2^1} 
	splits because each set partition has either $1$ or $2$ blocks.

	We show in full how to find the two basis elements 
	$X_\pi^{+}$, $X_\pi^{-}$
	in
	$\Hom_{A_2}((\mathbb{R}^{2})^{\otimes 2}, \mathbb{R}^{2})$
	that corresponds to the first set partition diagram $d_\pi$ in Figure \ref{matrix2,2^1}, where $\pi = \{1,2,3\}$, and state what they are for the other set partition diagrams in Figure \ref{matrix2,2^1A2}.

	First, we flatten the set partition diagram $d_\pi$. 
	Then we add a new top row consisting of $n=2$ vertices, and connect the lowest numbered vertex in each block $i$ of $\pi$ to vertex $i$ in the top row.
	Hence, we obtain the diagram $b_\pi$, which is
	
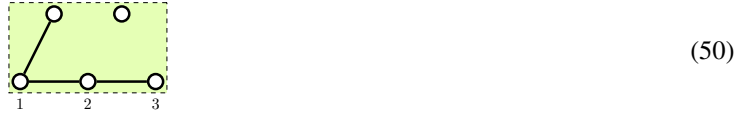
\begin{equation}
		\begin{aligned}
		\scalebox{0.6}{\input{flatorbit21sq1.tikz}}
		\end{aligned}
	\end{equation}
	Next, we attach a two-legged jellyfish to the top row of vertices, 
	giving	
	%to obtain
	\begin{equation}
		\scalebox{0.6}{\input{jellyfishorbit21sq1.tikz}}
	\end{equation}
	This is the diagram that is associated with the function $f_\pi$ that is given in Proposition \ref{fpiAnequiv}.

	To obtain the two $A_2$ orbits 
	%$O_\pi^{+}$ 
	$O_{A_2}((I_\pi^+, J_\pi^+))$
	and 
	%$O_\pi^{-}$ 
	$O_{A_2}((I_\pi^-, J_\pi^-))$
	corresponding to $\pi$, we apply Theorem \ref{fpimaps}.

	From Figure \ref{matrix2,2^1}, we see that
	$O_{S_2}((I_\pi, J_\pi)) = \{(1,1,1), (2,2,2)\}$,
	and so, we have that
	\begin{equation}
		f_\pi(e_{(1,1,1)})
		= \det \circ \, g_\pi(e_{(1,1,1)})
		= \det(e_{(1,1)} + e_{(1,2)})
		= +1
	\end{equation}
	and
	\begin{equation}
		f_\pi(e_{(2,2,2)})
		= \det \circ \, g_\pi(e_{(2,2,2)})
		= \det(e_{(2,1)} + e_{(2,2)})
		= -1
	\end{equation}
	Hence, by
	(\ref{Opiplus}) and
	(\ref{Opiminus}),
	we have that
	\begin{equation}
		%O_\pi^{+}
		O_{A_2}((I_\pi^+, J_\pi^+))
		= \{(1,1,1)\}
	\end{equation}
	and
	\begin{equation}
		%O_\pi^{-}
		O_{A_2}((I_\pi^-, J_\pi^-))
		= \{(2,2,2)\}
	\end{equation}
	and so,
	by (\ref{Xpiplus}) and
	(\ref{Xpiminus}),
	%we have that
	\begin{equation}
		X_\pi^{+}
		= E_{(1 \mid 1,1)}
	\end{equation}
	and
	\begin{equation}
		X_\pi^{-}
		= E_{(2 \mid 2,2)}
	\end{equation}

Consequently, from the matrices given in Figure \ref{matrix2,2^1A2},
%we have that 
the weight matrix for an $A_2$--equivariant linear layer function from $(\mathbb{R}^{2})^{\otimes 2}$ to $\mathbb{R}^{2}$
is of the form
\begin{equation}
	\NiceMatrixOptions{code-for-first-row = \scriptstyle \color{blue},
                   	   code-for-first-col = \scriptstyle \color{blue}
	}
	\begin{bNiceArray}{*{2}{c} | *{2}{c}}[first-row,first-col]
				& 1,1 		& 1,2	& 2,1	& 2,2 \\
		1		& \lambda_1	& \lambda_3	& \lambda_5	& \lambda_7	\\
		2		& \lambda_8	& \lambda_6	& \lambda_4	& \lambda_2	
	\end{bNiceArray}
\end{equation}
for scalars $\lambda_1, \dots, \lambda_8 \in \mathbb{R}$.
	
	\begin{figure}[h]
\begin{center}
\begin{tblr}{
  colspec = {X[c,h]X[c]X[c]X[c]},
  stretch = 0,
  rowsep = 5pt,
  hlines = {1pt},
  vlines = {1pt},
}
	{Set Partition Diagram \\ $d_\pi$} 	& {Partition \\ $\pi$} 	& 
	{Standard Basis Element \\ $X_\pi^{+}$} & 
	{Standard Basis Element \\ $X_\pi^{-}$} \\
	\scalebox{0.6}{\input{orbit21sq1.tikz}} &  $\{1, 2, 3\}$ &
	\scalebox{0.75}{
	$
	\NiceMatrixOptions{code-for-first-row = \scriptstyle \color{blue},
                   	   code-for-first-col = \scriptstyle \color{blue}
	}
	\begin{bNiceArray}{*{2}{c} | *{2}{c}}[first-row,first-col]
				& 1,1 		& 1,2	& 2,1	& 2,2 \\
		1		& 1	& 0	& 0	& 0	\\
		2		& 0	& 0	& 0	& 0	
	\end{bNiceArray}
	$}
	& 
	\scalebox{0.75}{
	$
	\NiceMatrixOptions{code-for-first-row = \scriptstyle \color{blue},
                   	   code-for-first-col = \scriptstyle \color{blue}
	}
	\begin{bNiceArray}{*{2}{c} | *{2}{c}}[first-row,first-col]
				& 1,1 		& 1,2	& 2,1	& 2,2 \\
		1		& 0	& 0	& 0	& 0	\\
		2		& 0	& 0	& 0	& 1	
	\end{bNiceArray}
	$}
	\\
	\scalebox{0.6}{\input{orbit21sq2.tikz}}	& $\{1, 2 \mid 3\}$ 	& 
	\scalebox{0.75}{
	$
	\NiceMatrixOptions{code-for-first-row = \scriptstyle \color{blue},
                   	   code-for-first-col = \scriptstyle \color{blue}
	}
	\begin{bNiceArray}{*{2}{c} | *{2}{c}}[first-row,first-col]
				& 1,1 		& 1,2	& 2,1	& 2,2 \\
		1		& 0	& 1	& 0	& 0	\\
		2		& 0	& 0	& 0	& 0	
	\end{bNiceArray}
	$} &
	\scalebox{0.75}{
	$
	\NiceMatrixOptions{code-for-first-row = \scriptstyle \color{blue},
                   	   code-for-first-col = \scriptstyle \color{blue}
	}
	\begin{bNiceArray}{*{2}{c} | *{2}{c}}[first-row,first-col]
				& 1,1 		& 1,2	& 2,1	& 2,2 \\
		1		& 0	& 0	& 0	& 0	\\
		2		& 0	& 0	& 1	& 0	
	\end{bNiceArray}
	$}
	\\
	\scalebox{0.6}{\input{orbit21sq3.tikz}}	& $\{1, 3 \mid 2\}$ 	& 
	\scalebox{0.75}{
	$
	\NiceMatrixOptions{code-for-first-row = \scriptstyle \color{blue},
                   	   code-for-first-col = \scriptstyle \color{blue}
	}
	\begin{bNiceArray}{*{2}{c} | *{2}{c}}[first-row,first-col]
				& 1,1 		& 1,2	& 2,1	& 2,2 \\
		1		& 0	& 0	& 1	& 0	\\
		2		& 0	& 0	& 0	& 0	
	\end{bNiceArray}
	$} &
	\scalebox{0.75}{
	$
	\NiceMatrixOptions{code-for-first-row = \scriptstyle \color{blue},
                   	   code-for-first-col = \scriptstyle \color{blue}
	}
	\begin{bNiceArray}{*{2}{c} | *{2}{c}}[first-row,first-col]
				& 1,1 		& 1,2	& 2,1	& 2,2 \\
		1		& 0	& 0	& 0	& 0	\\
		2		& 0	& 1	& 0	& 0	
	\end{bNiceArray}
	$}
	\\
	\scalebox{0.6}{\input{orbit21sq4.tikz}}	& $\{1 \mid 2, 3\}$ 	& 
	\scalebox{0.75}{
	$
	\NiceMatrixOptions{code-for-first-row = \scriptstyle \color{blue},
                   	   code-for-first-col = \scriptstyle \color{blue}
	}
	\begin{bNiceArray}{*{2}{c} | *{2}{c}}[first-row,first-col]
				& 1,1 		& 1,2	& 2,1	& 2,2 \\
		1		& 0	& 0	& 0	& 1	\\
		2		& 0	& 0	& 0	& 0	
	\end{bNiceArray}
	$} &
	\scalebox{0.75}{
	$
	\NiceMatrixOptions{code-for-first-row = \scriptstyle \color{blue},
                   	   code-for-first-col = \scriptstyle \color{blue}
	}
	\begin{bNiceArray}{*{2}{c} | *{2}{c}}[first-row,first-col]
				& 1,1 		& 1,2	& 2,1	& 2,2 \\
		1		& 0	& 0	& 0	& 0	\\
		2		& 1	& 0	& 0	& 0	
	\end{bNiceArray}
	$}
\end{tblr}
	\caption{
		By considering which set partitions in $\Pi_{1+2,2}$ split, we obtain a basis of
	$\Hom_{A_2}((\mathbb{R}^{2})^{\otimes 2}, \mathbb{R}^{2})$.}
  	\label{matrix2,2^1A2}
	\end{center}
\end{figure}

Note that if $n=3$, then we also need to consider the set partition diagram
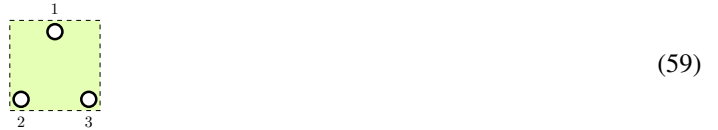
\begin{equation}
	\begin{aligned}
	\scalebox{0.6}{\input{orbit21sq5.tikz}}
	\end{aligned}
\end{equation}
corresponding to the set partition $\{1 \mid 2 \mid 3\}$, since this is a valid set partition in $\Pi_{1+2,3}$. 
In this case,
only the set partition $\{1,2,3\}$ does not split. 
Hence, while the dimension of 
$
	\Hom_{S_3}((\mathbb{R}^{3})^{\otimes 2}, \mathbb{R}^{3})
$
is $5$,
the dimension of 
$
	\Hom_{A_3}((\mathbb{R}^{3})^{\otimes 2}, \mathbb{R}^{3})
$
is $9$.

However, if $n \geq 4$, then none of the set partitions in $\Pi_{1+2,n}$
split, and so, in this case,
\begin{equation}
	\Hom_{S_n}((\mathbb{R}^{n})^{\otimes 2}, \mathbb{R}^{n})
	=
	\Hom_{A_n}((\mathbb{R}^{n})^{\otimes 2}, \mathbb{R}^{n})
\end{equation}
Consequently, the basis $\{X_\pi\}$, of size $5$, is the same for each space.

\section{Limitations and Feasibility}

It is important to acknowledge that given the
current limitations of hardware, there will be some challenges when
implementing the neural networks that are discussed in this paper.
In particular, significant engineering efforts will be needed to achieve the required scale because storing high-order tensors in memory
%, a crucial element in such neural networks, 
is not a straightforward task. 
This was demonstrated by Kondor et al.~\yrcite{clebschgordan}, who had to develop custom CUDA kernels in order to implement their tensor product based neural networks. 
Nevertheless, we anticipate that with the increasing availability of computing power, higher-order group equivariant neural networks will become more prevalent in practical applications. 
Notably, while the dimension of tensor power spaces increases exponentially with their order, the dimension of the space of equivariant maps between such tensor power spaces is often much smaller, and the corresponding matrices are typically sparse. 
Therefore, while storing these matrices may present some technical difficulties, it should be feasible with the current computing power that is available.

\section{Equivariance to Local Symmetries}

We can extend our results to looking at linear layer functions that are equivariant to a direct product of alternating groups; 
that is, we can construct neural networks that are equivariant to local symmetries.
The case for an external tensor product of order $1$ tensors can be found in \cite{maron20a}; below, we show the equivariance for any tensors of any order.

%Specifically, 
We wish to find a basis for 
\begin{equation} \label{genericHomspace}
	%\Hom_{A_{n_1} \times \dots \times A_{n_p}}(V, W)
	\Hom_{A_{n_1} \times \dots \times A_{n_p}}(
	(\mathbb{R}^{n_1})^{\otimes {k_1}} \boxtimes \dots \boxtimes (\mathbb{R}^{n_p})^{\otimes {k_p}},
	(\mathbb{R}^{n_1})^{\otimes {l_1}} \boxtimes \dots \boxtimes (\mathbb{R}^{n_p})^{\otimes {l_p}})
\end{equation}
%where 
%\begin{equation}
	%V \coloneqq (\mathbb{R}^{n_1})^{\otimes {k_1}} \boxtimes \dots \boxtimes (\mathbb{R}^{n_p})^{\otimes {k_p}}
%\end{equation}
%\begin{equation}
	%W \coloneqq (\mathbb{R}^{n_1})^{\otimes {l_1}} \boxtimes \dots \boxtimes (\mathbb{R}^{n_p})^{\otimes {l_p}}
%\end{equation}
%and 
where $\boxtimes$ is the external tensor product.

The $\Hom$-space given in (\ref{genericHomspace}) is isomorphic to
\begin{equation} \label{tensorgenericHomspace}
	\bigotimes_{r=1}^{p} \Hom_{A_{n_r}}((\mathbb{R}^{n_r})^{\otimes k_r}, (\mathbb{R}^{n_r})^{\otimes l_r})
\end{equation}

As Theorem \ref{mainAnthm} gives a basis for each individual $\Hom$--space in (\ref{tensorgenericHomspace}), we can obtain a basis for the overall $\Hom$--space (\ref{genericHomspace}) by forming all possible $p$--length Kronecker products of basis elements in the standard way for tensor product spaces.

%%%%%%%%%%%%%%%%%%%%%%%%%%%%%%%%%%%%%%%%%%%%%%%%%%%%%%%%%%%%%%%%%%%%%%%%%%%%%%%
%%%%%%%%%%%%%%%%%%%%%%%%%%%%%%%%%%%%%%%%%%%%%%%%%%%%%%%%%%%%%%%%%%%%%%%%%%%%%%%

\end{document}

%% file: mystyle.tikzstyles
\tikzstyle{node}=[fill=white, draw=black, shape=circle, minimum size=1mm, ultra thick]
\tikzstyle{small box}=[fill=white, draw=black, shape=rectangle, minimum height=0.5cm, minimum width=0.5cm, ultra thick]
\tikzstyle{weyl}=[fill=white, draw={rgb,255: red,0; green,0; blue,109}, shape=rectangle, minimum height=0.5cm, minimum width=0.5cm]
\tikzstyle{filled_node}=[fill=black, draw=black, shape=circle, minimum size=1mm, ultra thick]

\tikzstyle{thick}=[-, ultra thick]
\tikzstyle{blue_thick}=[-, ultra thick, draw=blue]
\tikzstyle{dashes}=[-, dashed, draw={rgb,255: red,191; green,191; blue,191}, dash pattern=on 2mm off 1mm, fill={rgb,255: red,244; green,228; blue,0}]
\tikzstyle{thick_arrow}=[ultra thick, ->]
\tikzstyle{dash_1}=[-, dashed]
\tikzstyle{dash_2}=[-, dashed, fill={rgb,255: red,246; green,235; blue,255}]
\tikzstyle{dash_3}=[-, dashed, fill={rgb,255: red,229; green,255; blue,181}]
\tikzstyle{dash_4}=[-, dashed, fill={rgb,255: red,255; green,209; blue,153}]
\tikzstyle{red_thick}=[-, ultra thick, draw=red]
\tikzstyle{dash_5}=[-, dashed, fill={rgb,255: red,225; green,255; blue,254}]
\tikzstyle{jellyfish}=[-, ultra thick, fill={rgb,255: red,34; green,48; blue,255}]

%% file: figures/partklelement1.tikz
\begin{tikzpicture}
	\begin{pgfonlayer}{nodelayer}
		\node [style=none] (10) at (1, -5) {$4$};
		\node [style=none] (11) at (4, -5) {$5$};
		\node [style=none] (12) at (7, -5) {$6$};
		\node [style=none] (13) at (10, -5) {$7$};
		\node [style=none] (14) at (13, -5) {$8$};
		\node [style=none] (15) at (4, 1) {$1$};
		\node [style=none] (16) at (7, 1) {$2$};
		\node [style=none] (17) at (10, 1) {$3$};
		\node [style=none] (18) at (0.5, 0.5) {};
		\node [style=none] (19) at (0.5, -4.5) {};
		\node [style=none] (20) at (13.5, -4.5) {};
		\node [style=none] (21) at (13.5, 0.5) {};
		\node [style=node] (22) at (10, 0) {};
		\node [style=node] (23) at (7, 0) {};
		\node [style=node] (24) at (4, 0) {};
		\node [style=node] (25) at (1, -4) {};
		\node [style=node] (26) at (4, -4) {};
		\node [style=node] (27) at (7, -4) {};
		\node [style=node] (28) at (10, -4) {};
		\node [style=node] (29) at (13, -4) {};
	\end{pgfonlayer}
	\begin{pgfonlayer}{edgelayer}
		\draw [style={dash_3}] (21.center)
			 to (18.center)
			 to (19.center)
			 to (20.center)
			 to cycle;
		\draw [style=thick] (26) to (24);
		\draw [style=thick] (25) to (23);
		\draw [style=thick] (27) to (22);
		\draw [style=thick, bend left=45, looseness=0.75] (27) to (29);
	\end{pgfonlayer}
\end{tikzpicture}

%% file: figures/partklelement2.tikz
\begin{tikzpicture}
	\begin{pgfonlayer}{nodelayer}
		\node [style=none] (0) at (10, -5) {$4$};
		\node [style=none] (1) at (13, -5) {$5$};
		\node [style=none] (2) at (16, -5) {$6$};
		\node [style=none] (3) at (19, -5) {$7$};
		\node [style=none] (4) at (22, -5) {$8$};
		\node [style=none] (5) at (1, -5) {$1$};
		\node [style=none] (6) at (4, -5) {$2$};
		\node [style=none] (7) at (7, -5) {$3$};
		\node [style=none] (8) at (0.5, 0.5) {};
		\node [style=none] (9) at (0.5, -4.5) {};
		\node [style=none] (10) at (22.5, -4.5) {};
		\node [style=none] (11) at (22.5, 0.5) {};
		\node [style=node] (12) at (7, -4) {};
		\node [style=node] (13) at (4, -4) {};
		\node [style=node] (14) at (1, -4) {};
		\node [style=node] (15) at (10, -4) {};
		\node [style=node] (16) at (13, -4) {};
		\node [style=node] (17) at (16, -4) {};
		\node [style=node] (18) at (19, -4) {};
		\node [style=node] (19) at (22, -4) {};
		\node [style=none] (20) at (10, 1) {};
		\node [style=none] (21) at (13, 1) {};
		\node [style=none] (22) at (16, 1) {};
	\end{pgfonlayer}
	\begin{pgfonlayer}{edgelayer}
		\draw [style={dash_3}] (11.center)
			 to (8.center)
			 to (9.center)
			 to (10.center)
			 to cycle;
		\draw [style=thick, bend right=45, looseness=0.75] (16) to (14);
		\draw [style=thick, bend left=315, looseness=0.75] (15) to (13);
		\draw [style=thick, bend right=45, looseness=0.75] (17) to (12);
		\draw [style=thick, bend left=45, looseness=0.75] (17) to (19);
	\end{pgfonlayer}
\end{tikzpicture}

%% file: figures/jellyfishex.tikz
\begin{tikzpicture}
	\begin{pgfonlayer}{nodelayer}
		\node [style=node] (41) at (4, -4) {};
		\node [style=node] (43) at (0, -4) {};
		\node [style=node] (45) at (-4, -4) {};
		\node [style=none] (46) at (-3, 0) {};
		\node [style=none] (47) at (3, 0) {};
		\node [style=none] (48) at (0, 2.75) {};
		\node [style=none] (49) at (0, 0) {};
		\node [style=none] (51) at (-2, 0) {};
		\node [style=none] (53) at (2, 0) {};
	\end{pgfonlayer}
	\begin{pgfonlayer}{edgelayer}
		\draw [style=jellyfish] (48.center)
			 to [bend right=45, looseness=0.75] (46.center)
			 to (47.center)
			 to [bend right=45, looseness=0.75] cycle;
		\draw [style=thick] (45) to (51.center);
		\draw [style=thick] (43) to (49.center);
		\draw [style=thick] (41) to (53.center);
	\end{pgfonlayer}
\end{tikzpicture}

%% file: figures/jellyfishattached.tikz
\begin{tikzpicture}
	\begin{pgfonlayer}{nodelayer}
		\node [style=none] (0) at (1, -5) {$4$};
		\node [style=none] (1) at (4, -5) {$5$};
		\node [style=none] (2) at (7, -5) {$6$};
		\node [style=none] (3) at (10, -5) {$7$};
		\node [style=none] (4) at (13, -5) {$8$};
		\node [style=none] (5) at (-8, -5) {$1$};
		\node [style=none] (6) at (-5, -5) {$2$};
		\node [style=none] (7) at (-2, -5) {$3$};
		\node [style=none] (8) at (-8.5, 0.5) {};
		\node [style=none] (9) at (-8.5, -4.5) {};
		\node [style=none] (10) at (13.5, -4.5) {};
		\node [style=none] (11) at (13.5, 0.5) {};
		\node [style=node] (12) at (-2, -4) {};
		\node [style=node] (13) at (-5, -4) {};
		\node [style=node] (14) at (-8, -4) {};
		\node [style=node] (15) at (1, -4) {};
		\node [style=node] (16) at (4, -4) {};
		\node [style=node] (17) at (7, -4) {};
		\node [style=node] (18) at (10, -4) {};
		\node [style=node] (19) at (13, -4) {};
		\node [style=node] (20) at (8.5, 0) {};
		\node [style=node] (21) at (5.5, 0) {};
		\node [style=node] (22) at (2.5, 0) {};
		\node [style=node] (23) at (-0.5, 0) {};
		\node [style=node] (24) at (-3.5, 0) {};
		\node [style=none] (26) at (-0.5, 5) {};
		\node [style=none] (27) at (5.5, 5) {};
		\node [style=none] (28) at (2.5, 7.75) {};
		\node [style=none] (29) at (2.5, 5) {};
		\node [style=none] (30) at (1.5, 5) {};
		\node [style=none] (31) at (0.5, 5) {};
		\node [style=none] (32) at (3.5, 5) {};
		\node [style=none] (33) at (4.5, 5) {};
	\end{pgfonlayer}
	\begin{pgfonlayer}{edgelayer}
		\draw [style={dash_3}] (11.center)
			 to (8.center)
			 to (9.center)
			 to (10.center)
			 to cycle;
		\draw [style=thick, bend right=45, looseness=0.75] (16) to (14);
		\draw [style=thick, bend left=315, looseness=0.75] (15) to (13);
		\draw [style=thick, bend right=45, looseness=0.75] (17) to (12);
		\draw [style=thick, bend left=45, looseness=0.75] (17) to (19);
		\draw [style=jellyfish] (28.center)
			 to [bend right=45, looseness=0.75] (26.center)
			 to (27.center)
			 to [bend right=45, looseness=0.75] cycle;
		\draw [style=thick] (24) to (31.center);
		\draw [style=thick] (23) to (30.center);
		\draw [style=thick] (22) to (29.center);
		\draw [style=thick] (21) to (32.center);
		\draw [style=thick] (20) to (33.center);
		\draw [style=thick] (14) to (24);
		\draw [style=thick] (13) to (23);
		\draw [style=thick] (12) to (22);
		\draw [style=thick] (18) to (21);
	\end{pgfonlayer}
\end{tikzpicture}

%% file: figures/orbit21sq1.tikz
\begin{tikzpicture}
	\begin{pgfonlayer}{nodelayer}
		\node [style=none] (0) at (2.5, 1) {$1$};
		\node [style=node] (2) at (1, -3) {};
		\node [style=node] (3) at (2.5, 0) {};
		\node [style=node] (4) at (4, -3) {};
		\node [style=none] (5) at (1, -4) {$2$};
		\node [style=none] (6) at (4, -4) {$3$};
		\node [style=none] (8) at (0.5, 0.5) {};
		\node [style=none] (9) at (0.5, -3.5) {};
		\node [style=none] (10) at (4.5, -3.5) {};
		\node [style=none] (11) at (4.5, 0.5) {};
	\end{pgfonlayer}
	\begin{pgfonlayer}{edgelayer}
		\draw [style={dash_3}] (11.center)
			 to (8.center)
			 to (9.center)
			 to (10.center)
			 to cycle;
		\draw [style=thick] (2) to (3);
		\draw [style=thick] (3) to (4);
		\draw [style=thick] (4) to (2);
	\end{pgfonlayer}
\end{tikzpicture}

%% file: figures/orbit21sq2.tikz
\begin{tikzpicture}
	\begin{pgfonlayer}{nodelayer}
		\node [style=none] (0) at (2.5, 1) {$1$};
		\node [style=node] (1) at (1, -3) {};
		\node [style=node] (2) at (2.5, 0) {};
		\node [style=node] (3) at (4, -3) {};
		\node [style=none] (4) at (1, -4) {$2$};
		\node [style=none] (5) at (4, -4) {$3$};
		\node [style=none] (6) at (0.5, 0.5) {};
		\node [style=none] (7) at (0.5, -3.5) {};
		\node [style=none] (8) at (4.5, -3.5) {};
		\node [style=none] (9) at (4.5, 0.5) {};
	\end{pgfonlayer}
	\begin{pgfonlayer}{edgelayer}
		\draw [style={dash_3}] (9.center)
			 to (6.center)
			 to (7.center)
			 to (8.center)
			 to cycle;
		\draw [style=thick] (1) to (2);
	\end{pgfonlayer}
\end{tikzpicture}

%% file: figures/orbit21sq3.tikz
\begin{tikzpicture}
	\begin{pgfonlayer}{nodelayer}
		\node [style=none] (0) at (2.5, 1) {$1$};
		\node [style=node] (1) at (1, -3) {};
		\node [style=node] (2) at (2.5, 0) {};
		\node [style=node] (3) at (4, -3) {};
		\node [style=none] (4) at (1, -4) {$2$};
		\node [style=none] (5) at (4, -4) {$3$};
		\node [style=none] (6) at (0.5, 0.5) {};
		\node [style=none] (7) at (0.5, -3.5) {};
		\node [style=none] (8) at (4.5, -3.5) {};
		\node [style=none] (9) at (4.5, 0.5) {};
	\end{pgfonlayer}
	\begin{pgfonlayer}{edgelayer}
		\draw [style={dash_3}] (9.center)
			 to (6.center)
			 to (7.center)
			 to (8.center)
			 to cycle;
		\draw [style=thick] (2) to (3);
	\end{pgfonlayer}
\end{tikzpicture}

%% file: figures/orbit21sq4.tikz
\begin{tikzpicture}
	\begin{pgfonlayer}{nodelayer}
		\node [style=none] (0) at (2.5, 1) {$1$};
		\node [style=node] (1) at (1, -3) {};
		\node [style=node] (2) at (2.5, 0) {};
		\node [style=node] (3) at (4, -3) {};
		\node [style=none] (4) at (1, -4) {$2$};
		\node [style=none] (5) at (4, -4) {$3$};
		\node [style=none] (6) at (0.5, 0.5) {};
		\node [style=none] (7) at (0.5, -3.5) {};
		\node [style=none] (8) at (4.5, -3.5) {};
		\node [style=none] (9) at (4.5, 0.5) {};
	\end{pgfonlayer}
	\begin{pgfonlayer}{edgelayer}
		\draw [style={dash_3}] (9.center)
			 to (6.center)
			 to (7.center)
			 to (8.center)
			 to cycle;
		\draw [style=thick] (3) to (1);
	\end{pgfonlayer}
\end{tikzpicture}

%% file: figures/flatorbit21sq1.tikz
\begin{tikzpicture}
	\begin{pgfonlayer}{nodelayer}
		\node [style=none] (0) at (-2, -4) {$1$};
		\node [style=node] (1) at (1, -3) {};
		\node [style=node] (2) at (-2, -3) {};
		\node [style=node] (3) at (4, -3) {};
		\node [style=none] (4) at (1, -4) {$2$};
		\node [style=none] (5) at (4, -4) {$3$};
		\node [style=none] (6) at (-2.5, 0.5) {};
		\node [style=none] (7) at (-2.5, -3.5) {};
		\node [style=none] (8) at (4.5, -3.5) {};
		\node [style=none] (9) at (4.5, 0.5) {};
		\node [style=node] (10) at (-0.5, 0) {};
		\node [style=node] (11) at (2.5, 0) {};
	\end{pgfonlayer}
	\begin{pgfonlayer}{edgelayer}
		\draw [style={dash_3}] (9.center)
			 to (6.center)
			 to (7.center)
			 to (8.center)
			 to cycle;
		\draw [style=thick] (2) to (1);
		\draw [style=thick] (1) to (3);
		\draw [style=thick] (2) to (10);
	\end{pgfonlayer}
\end{tikzpicture}

%% file: figures/jellyfishorbit21sq1.tikz
\begin{tikzpicture}
	\begin{pgfonlayer}{nodelayer}
		\node [style=none] (0) at (-2, -4) {$1$};
		\node [style=node] (1) at (1, -3) {};
		\node [style=node] (2) at (-2, -3) {};
		\node [style=node] (3) at (4, -3) {};
		\node [style=none] (4) at (1, -4) {$2$};
		\node [style=none] (5) at (4, -4) {$3$};
		\node [style=none] (6) at (-2.5, 0.5) {};
		\node [style=none] (7) at (-2.5, -3.5) {};
		\node [style=none] (8) at (4.5, -3.5) {};
		\node [style=none] (9) at (4.5, 0.5) {};
		\node [style=node] (10) at (-0.5, 0) {};
		\node [style=node] (11) at (2.5, 0) {};
		\node [style=none] (12) at (-2, 3) {};
		\node [style=none] (13) at (4, 3) {};
		\node [style=none] (14) at (1, 5.75) {};
		\node [style=none] (15) at (1, 3) {};
		\node [style=none] (16) at (0, 3) {};
		\node [style=none] (17) at (-1, 3) {};
		\node [style=none] (18) at (2, 3) {};
		\node [style=none] (19) at (3, 3) {};
	\end{pgfonlayer}
	\begin{pgfonlayer}{edgelayer}
		\draw [style={dash_3}] (9.center)
			 to (6.center)
			 to (7.center)
			 to (8.center)
			 to cycle;
		\draw [style=thick] (2) to (1);
		\draw [style=thick] (1) to (3);
		\draw [style=thick] (2) to (10);
		\draw [style=jellyfish] (14.center)
			 to [bend right=45, looseness=0.75] (12.center)
			 to (13.center)
			 to [bend right=45, looseness=0.75] cycle;
		\draw [style=thick] (10) to (16.center);
		\draw [style=thick] (11) to (18.center);
	\end{pgfonlayer}
\end{tikzpicture}

%% file: figures/orbit21sq5.tikz
\begin{tikzpicture}
	\begin{pgfonlayer}{nodelayer}
		\node [style=none] (0) at (2.5, 1) {$1$};
		\node [style=node] (1) at (1, -3) {};
		\node [style=node] (2) at (2.5, 0) {};
		\node [style=node] (3) at (4, -3) {};
		\node [style=none] (4) at (1, -4) {$2$};
		\node [style=none] (5) at (4, -4) {$3$};
		\node [style=none] (6) at (0.5, 0.5) {};
		\node [style=none] (7) at (0.5, -3.5) {};
		\node [style=none] (8) at (4.5, -3.5) {};
		\node [style=none] (9) at (4.5, 0.5) {};
	\end{pgfonlayer}
	\begin{pgfonlayer}{edgelayer}
		\draw [style={dash_3}] (9.center)
			 to (6.center)
			 to (7.center)
			 to (8.center)
			 to cycle;
	\end{pgfonlayer}
\end{tikzpicture}

%% file: jellyfish.bbl
\begin{thebibliography}{27}
\providecommand{\natexlab}[1]{#1}
\providecommand{\url}[1]{\texttt{#1}}
\expandafter\ifx\csname urlstyle\endcsname\relax
  \providecommand{\doi}[1]{doi: #1}\else
  \providecommand{\doi}{doi: \begingroup \urlstyle{rm}\Url}\fi

\bibitem[Barcelo \& Ram(1997)Barcelo and Ram]{barceloram}
Barcelo, H. and Ram, A.
\newblock Combinatorial {R}epresentation {T}heory, 1997.
\newblock \texttt{arXiv:math/9707221}.

\bibitem[Benkart \& Halverson(2019{\natexlab{a}})Benkart and
  Halverson]{BenHal1}
Benkart, G. and Halverson, T.
\newblock Partition algebras {$P_k(n)$} with $2k>n$ and the fundamental
  theorems of invariant theory for the symmetric group {$S_n$}.
\newblock \emph{J. London Math. Soc}, 99\penalty0 (2):\penalty0 194--224,
  2019{\natexlab{a}}.
\newblock URL \url{https://doi.org/10.1112/jlms.12175}.

\bibitem[Benkart \& Halverson(2019{\natexlab{b}})Benkart and
  Halverson]{BenHal2}
Benkart, G. and Halverson, T.
\newblock Partition {A}lgebras and the {I}nvariant {T}heory of the {S}ymmetric
  {G}roup.
\newblock In Barcelo, H., Karaali, G., and Orellana, R. (eds.), \emph{Recent
  Trends in Algebraic Combinatorics}, volume~16 of \emph{Association for Women
  in Mathematics Series}, pp.\  1--41. Springer, 2019{\natexlab{b}}.
\newblock URL \url{https://doi.org/10.1007/978-3-030-05141-9}.

\bibitem[Benkart et~al.(2017)Benkart, Halverson, and Harman]{BHH}
Benkart, G., Halverson, T., and Harman, N.
\newblock Dimensions of irreducible modules for partition algebras and tensor
  power multiplicities for symmetric and alternating groups.
\newblock \emph{J. Algebraic Combin.}, 46\penalty0 (1):\penalty0 77--108, 2017.
\newblock ISSN 0925-9899.
\newblock URL \url{https://doi.org/10.1007/s10801-017-0748-4}.

\bibitem[Bloss(2005)]{bloss}
Bloss, M.
\newblock The {P}artition {A}lgebra as a {C}entralizer {A}lgebra of the
  {A}lternating {G}roup.
\newblock \emph{Communications in Algebra}, 33:\penalty0 2219--2229, 2005.
\newblock URL \url{https://doi.org/10.1081/AGB-200063579}.

\bibitem[Brauer(1937)]{Brauer}
Brauer, R.
\newblock On {A}lgebras {W}hich {A}re {C}onnected with the {S}emisimple
  {C}ontinuous {G}roups.
\newblock \emph{Ann. Math.}, 38:\penalty0 857--872, 1937.
\newblock ISSN 0003486X.
\newblock URL \url{https://doi.org/10.2307/1968843}.

\bibitem[Ceccherini-Silberstein et~al.(2010)Ceccherini-Silberstein, Scarabotti,
  and Tolli]{tolli}
Ceccherini-Silberstein, T., Scarabotti, F., and Tolli, F.
\newblock \emph{Representation Theory of the Symmetric Groups}.
\newblock Cambridge University Press, 2010.

\bibitem[Cohen \& Welling(2016)Cohen and Welling]{cohenwelling}
Cohen, T. and Welling, M.
\newblock Group {E}quivariant {C}onvolutional {N}etworks.
\newblock In Balcan, M.~F. and Weinberger, K.~Q. (eds.), \emph{Proceedings of
  The 33rd International Conference on Machine Learning}, volume~48 of
  \emph{Proceedings of Machine Learning Research}, pp.\  2990--2999, New York,
  New York, USA, 20--22 Jun 2016. PMLR.
\newblock URL \url{https://proceedings.mlr.press/v48/cohenc16.html}.

\bibitem[Cohen et~al.(2019)Cohen, Weiler, Kicanaoglu, and Welling]{cohen19d}
Cohen, T., Weiler, M., Kicanaoglu, B., and Welling, M.
\newblock Gauge {E}quivariant {C}onvolutional {N}etworks and the {I}cosahedral
  {CNN}.
\newblock In Chaudhuri, K. and Salakhutdinov, R. (eds.), \emph{Proceedings of
  the 36th International Conference on Machine Learning}, volume~97 of
  \emph{Proceedings of Machine Learning Research}, pp.\  1321--1330. PMLR,
  09--15 Jun 2019.
\newblock URL \url{https://proceedings.mlr.press/v97/cohen19d.html}.

\bibitem[Comes(2020)]{comes}
Comes, J.
\newblock Jellyfish {P}artition {C}ategories.
\newblock \emph{Algebr Represent Theor}, 23:\penalty0 327--347, 2020.
\newblock URL \url{https://doi.org/10.1007/s10468-018-09851-7}.

\bibitem[Goodman \& Wallach(2009)Goodman and Wallach]{goodman}
Goodman, R. and Wallach, N.~R.
\newblock \emph{Symmetry, Representations and Invariants}.
\newblock Springer, 2009.

\bibitem[Halverson \& Ram(2005)Halverson and Ram]{HalRam}
Halverson, T. and Ram, A.
\newblock Partition algebras.
\newblock \emph{European J. Combin.}, 26\penalty0 (6):\penalty0 869--921, 2005.
\newblock ISSN 0195-6698.
\newblock URL \url{https://doi.org/10.1016/j.ejc.2004.06.005}.

\bibitem[Halverson \& Ram(2022)Halverson and Ram]{HalRam2022}
Halverson, T. and Ram, A.
\newblock Gems from the {W}ork of {G}eorgia {B}enkart.
\newblock \emph{Notices of the American Mathematical Society}, 69\penalty0
  (3):\penalty0 375--384, March 2022.
\newblock URL \url{https://doi.org/10.1090/noti2447}.

\bibitem[Jones(1994)]{Jones}
Jones, V.
\newblock The {P}otts model and the symmetric group.
\newblock In Araki, H., Kawahigashi, Y., and Kosaki, H. (eds.),
  \emph{Subfactors: Proceedings of the {T}aniguchi Symposium on Operator
  Algebras ({K}yuzeso, 1993)}, pp.\  259--267. World Scientific, 1994.

\bibitem[Kicki et~al.(2020)Kicki, Ozay, and Skrzypczynski]{kicki}
Kicki, P., Ozay, M., and Skrzypczynski, P.
\newblock A {C}omputationally {E}fficient {N}eural {N}etwork {I}nvariant to the
  {A}ction of {S}ymmetry {S}ubgroups, 2020.
\newblock \texttt{arXiv:2002.07528}.

\bibitem[Kondor et~al.(2018)Kondor, Lin, and Trivedi]{clebschgordan}
Kondor, R., Lin, Z., and Trivedi, S.
\newblock Clebsch\textendash {G}ordan {N}ets: a {F}ully {F}ourier {S}pace
  {S}pherical {C}onvolutional {N}eural {N}etwork.
\newblock In Bengio, S., Wallach, H., Larochelle, H., Grauman, K.,
  Cesa-Bianchi, N., and Garnett, R. (eds.), \emph{Advances in Neural
  Information Processing Systems}, volume~31. Curran Associates, Inc., 2018.
\newblock URL
  \url{https://proceedings.neurips.cc/paper/2018/file/a3fc981af450752046be179185ebc8b5-Paper.pdf}.

\bibitem[Lim \& Nelson(2022)Lim and Nelson]{lim}
Lim, L.-H. and Nelson, B.~J.
\newblock What is an equivariant neural network?, 2022.
\newblock \texttt{arXiv:2205.07362}.

\bibitem[Maron et~al.(2019{\natexlab{a}})Maron, Ben-Hamu, Shamir, and
  Lipman]{maron2018}
Maron, H., Ben-Hamu, H., Shamir, N., and Lipman, Y.
\newblock Invariant and {E}quivariant {G}raph {N}etworks.
\newblock In \emph{International Conference on Learning Representations},
  2019{\natexlab{a}}.
\newblock URL \url{https://openreview.net/forum?id=Syx72jC9tm}.

\bibitem[Maron et~al.(2019{\natexlab{b}})Maron, Fetaya, Segol, and
  Lipman]{maron19a}
Maron, H., Fetaya, E., Segol, N., and Lipman, Y.
\newblock On the {U}niversality of {I}nvariant {N}etworks.
\newblock In Chaudhuri, K. and Salakhutdinov, R. (eds.), \emph{Proceedings of
  the 36th International Conference on Machine Learning}, volume~97 of
  \emph{Proceedings of Machine Learning Research}, pp.\  4363--4371. PMLR,
  09--15 Jun 2019{\natexlab{b}}.
\newblock URL \url{https://proceedings.mlr.press/v97/maron19a.html}.

\bibitem[Maron et~al.(2020)Maron, Litany, Chechik, and Fetaya]{maron20a}
Maron, H., Litany, O., Chechik, G., and Fetaya, E.
\newblock On {L}earning {S}ets of {S}ymmetric {E}lements.
\newblock In III, H.~D. and Singh, A. (eds.), \emph{Proceedings of the 37th
  International Conference on Machine Learning}, volume 119 of
  \emph{Proceedings of Machine Learning Research}, pp.\  6734--6744. PMLR,
  13--18 Jul 2020.
\newblock URL \url{https://proceedings.mlr.press/v119/maron20a.html}.

\bibitem[Pearce-Crump(2022{\natexlab{a}})]{pearcecrump}
Pearce-Crump, E.
\newblock {C}onnecting {P}ermutation {E}quivariant {N}eural {N}etworks and
  {P}artition {D}iagrams.
\newblock \texttt{arXiv:2212.08648}, 2022{\natexlab{a}}.

\bibitem[Pearce-Crump(2022{\natexlab{b}})]{pearcecrumpB}
Pearce-Crump, E.
\newblock {B}rauer's {G}roup {E}quivariant {N}eural {N}etworks.
\newblock \texttt{arXiv:2212.08630}, 2022{\natexlab{b}}.

\bibitem[Ravanbakhsh(2020)]{ravanbakhsh}
Ravanbakhsh, S.
\newblock Universal {E}quivariant {M}ultilayer {P}erceptrons.
\newblock In \emph{Proceedings of the 37th International Conference on Machine
  Learning, ICML 2020, 13-18 July 2020, Virtual Event}, volume 119 of
  \emph{Proceedings of Machine Learning Research}, pp.\  7996--8006. PMLR,
  2020.
\newblock URL \url{http://proceedings.mlr.press/v119/ravanbakhsh20a.html}.

\bibitem[Sagan(2000)]{sagan}
Sagan, B.~E.
\newblock \emph{The Symmetric Group: Representations, Combinatorial Algorithms,
  and Symmetric Functions}.
\newblock Springer, 2000.

\bibitem[Segal(2014)]{segal}
Segal, E.
\newblock {G}roup {R}epresentation {T}heory.
\newblock Course Notes for Imperial College London, 2014.

\bibitem[Zaheer et~al.(2017)Zaheer, Kottur, Ravanbakhsh, Poczos, Salakhutdinov,
  and Smola]{deepsets}
Zaheer, M., Kottur, S., Ravanbakhsh, S., Poczos, B., Salakhutdinov, R.~R., and
  Smola, A.~J.
\newblock Deep {S}ets.
\newblock In Guyon, I., Luxburg, U.~V., Bengio, S., Wallach, H., Fergus, R.,
  Vishwanathan, S., and Garnett, R. (eds.), \emph{Advances in Neural
  Information Processing Systems}, volume~30. Curran Associates, Inc., 2017.
\newblock URL
  \url{https://proceedings.neurips.cc/paper/2017/file/f22e4747da1aa27e363d86d40ff442fe-Paper.pdf}.

\bibitem[Zhang et~al.(2019)Zhang, Liwicki, Smith, and Cipolla]{zhang}
Zhang, C., Liwicki, S., Smith, W., and Cipolla, R.
\newblock Orientation-{A}ware {S}emantic {S}egmentation on {I}cosahedron
  {S}pheres.
\newblock In \emph{2019 IEEE/CVF International Conference on Computer Vision
  (ICCV)}, pp.\  3532--3540, Los Alamitos, CA, USA, nov 2019. IEEE Computer
  Society.
\newblock \doi{10.1109/ICCV.2019.00363}.
\newblock URL
  \url{https://doi.ieeecomputersociety.org/10.1109/ICCV.2019.00363}.

\end{thebibliography}
